\documentclass[11pt]{article}

\usepackage{natbib}
\usepackage{amsfonts}
\usepackage{amsthm}
\usepackage{amsmath}
\usepackage[english]{babel}
\usepackage[english=british]{csquotes}
\usepackage[capitalize,nameinlink]{cleveref}
\usepackage{tikz}
\usepackage{color}
\usepackage{xcolor}
\usepackage{soul}
\usepackage{microtype}
\def\Exp{{\mathbb{E}}}
\newcommand{\cov}{{\rm cov}}
\newcommand{\corr}{{\rm corr}}
\usepackage[shortlabels]{enumitem}
\setlist{noitemsep}
\usepackage{comment}
\usepackage{subcaption}
\usepackage{mathtools}

\newcommand{\R}{{\mathbb R}}

\newcommand{\cY}{{\cal Y}}
\newcommand{\cW}{{\cal W}}
\newcommand{\cX}{{\cal X}}
\newcommand{\cZ}{{\cal Z}}

\newcommand{\cM}{{\cal M}}
\newcommand{\cF}{{\cal F}}
\newcommand{\cG}{{\cal G}}

\newcommand{\id}{{\bf I}}

\newcommand{\bx}{{\bf x}}
\newcommand{\bX}{{\bf X}}
\newcommand{\bY}{{\bf Y}}
\newcommand{\bZ}{{\bf Z}}

\newcommand{\bN}{{\bf N}}

\newcommand{\TODO}[1]{\textcolor{red}{#1}}

\newcommand\independent{\protect\mathpalette{\protect\independenT}{\perp}}
\def\independenT#1#2{\mathrel{\rlap{$#1#2$}\mkern2mu{#1#2}}}
\newcommand{\ind}{\independent}

\definecolor{lightgray}{gray}{0.85}
\sethlcolor{Dandelion}

\usetikzlibrary{arrows,shapes,plotmarks,positioning}
\usetikzlibrary{decorations.markings}

\tikzset{>=stealth'} 
\tikzset{node distance=2cm}
\tikzstyle{graphnode} = 
   [circle,draw=black,minimum size=22pt,text centered,text
     width=22pt,inner sep=0pt] 
\tikzstyle{var}   =[graphnode,fill=white]
\tikzstyle{vardashed}   =[graphnode,draw=gray,fill=white]
\tikzstyle{obs}   =[graphnode,fill=black,text=white]
\tikzstyle{obsgrey}   =[graphnode,draw=white,fill=lightgray,text=black]
\tikzstyle{par}    =[graphnode,draw=white,fill=red,text=black] 
 \tikzstyle{crucial} =[graphnode,draw=white,fill=yellow,text=black] 
\tikzstyle{fac}   =[rectangle,draw=black,fill=black!25,minimum size=5pt]
\tikzstyle{facprior} =[rectangle,draw=black,fill=black,text=white,minimum size=5pt]
\tikzstyle{edge}  =[draw=white,double=black,very thick,-]
\tikzstyle{blueedge}  =[draw=white,double=blue,very thick,-]
\tikzstyle{rededge}  =[draw=white,double=red,very thick,-]
\tikzstyle{prior} =[rectangle, draw=black, fill=black, minimum size=
5pt, inner sep=0pt]
\tikzstyle{dirprior} = [circle, draw=black, fill=black, minimum
size=5pt, inner sep=0pt]

\tikzstyle{dot_node}=[draw=black,fill=black,shape=circle]

\newtheorem{theorem}{Theorem}
\newtheorem{lemma}{Lemma}
\newtheorem{definition}{Definition}
\newtheorem{example}{Example}
\newtheorem{remark}{Remark}

\usepackage{tikz}

\usetikzlibrary{arrows.meta,
	chains,
	positioning,
	shapes.geometric
}
\begin{document}

\title{Reinterpreting causal discovery as the task of predicting unobserved joint statistics}

\author{Dominik Janzing$^1$, Philipp M. Faller$^2$,  Leena Chennuru Vankadara$^1$\\
{\small 1) Amazon Research Tübingen, Germany } \\
{\small 2) Karlsruhe Institute of Technology, Germany }
}

\maketitle

\begin{abstract}
If $\bX,\bY,\bZ$ denote sets of random variables, two different data sources may
contain samples from $P_{\bX,\bY}$ and 
$P_{\bY,\bZ}$, respectively. 
We argue that causal discovery can help inferring properties of the `unobserved joint distributions' $P_{\bX,\bY,\bZ}$ or
$P_{\bX,\bZ}$. 
The properties may be
conditional independences (as in `integrative causal inference') or also quantitative statements about dependences. 

More generally, we  define 
a learning scenario where the input is a subset of variables and the label is some statistical property of that subset. Sets of jointly observed variables
define the training points, while unobserved sets are possible test points. 
To solve this learning task,  we infer, as an intermediate step, a causal model from the observations 
that then entails properties of unobserved sets. 
Accordingly, we can define the VC dimension of
a class of causal models and derive generalization bounds for the predictions. 

Here, causal discovery becomes more modest and better accessible to empirical tests than usual: rather than trying to find a causal hypothesis 
that is `true' 
(which is a problematic term when
it is unclear how to define interventions)
a causal hypothesis is {\it useful} whenever 
it correctly predicts statistical properties of unobserved joint distributions. 
This way, a sparse causal graph that omits weak influences may be more useful than a dense one (despite being less accurate) because it 
is able to reconstruct the full joint distribution from marginal distributions of smaller subsets.  

Within such a `pragmatic' application of causal discovery, some popular  heuristic approaches 
become justified in retrospect. It is, for instance, allowed to infer DAGs from partial correlations instead of conditional independences if the DAGs are only used to predict partial correlations. 

We further sketch why our pragmatic view on causality may even cover the usual meaning in terms of interventions and sketch why predicting the impact of interventions can sometimes also be phrased as a task of the above type.
\end{abstract}

\section{Introduction}
\label{sec:introduction}
The difficulty of inferring causal relations from purely observational data lies in the fact that observations drawn from a joint distribution $P_\bX$ with $\bX:=\{X_1,\dots X_n\}$ are supposed to imply statements about how the system behaves under {\it interventions} \citep{Pearl2000,Spirtes1993}. Specifically,
one may be interested in the new joint distribution induced by \textit{setting} 
a subset $\tilde{\bX} \subset \bX$ of the variables to some specific values. Under this interventional defintion of causality, assessing the performance of causal discovery algorithms is highly challenging, primarily due to the absence of datasets with established causal ground truth or even the causal equivalent of a validation set. This issue primarily stems from the fact that conducting experiments or interventions is often infeasible, impractical, unethical, or ill-defined to begin with. For example, it is unclear what it means to intervene on the age of a person or the Gross Domestic Product of a country {(see \citet{janzing2022phenomenological} and references therein for a more elaborate discussion on ill-definedness of interventions)}.


\paragraph{Utility of causal information without reference to interventions.} The utility of causal models goes beyond the sole objective of predicting system behavior under interventions. For example causal information can be useful in facilitating knowledge transfer across datasets from different distributions \citep{anticausal}. Among the numerous other ways in which causal models can be useful, we particularly emphasize the utility of causal models in predicting statistical properties of \textit{unobserved joint distributions},\footnote{By a slight abuse of terminology, we have referred to sets of variables that have not been observed together as `unobserved sets of variables'. However, note that although they have not been observed {\it jointly}, they usually have been observed individually as part of some other observed set).} by leveraging multiple heterogeneous datasets with overlapping variables. Assume we have access to datasets $D_1,\dots,D_k$ containing observations from
different, but overlapping sets $S_1,\dots,S_k\subset \{X_1, X_2, \cdots, X_n\}$ of variables. 
Joint causal models over the variables $\{X_1, X_2, \cdots, X_n\}$  learned from datasets $S_1,\dots,S_k$ then entail statistical properties such as conditional independences over subsets of variables for which no joint observations are available. For instance, various methods under the umbrella of integrative causal inference learn such joint causal models by first applying causal discovery algorithms independently to datasets $S_1,\dots,S_k$ to learn \textit{marginal causal models} and subsequently, constructing a joint causal model that is consistent with the marginal causal models \citep{danks2005scientific, danks2008integrating, claassen2010causal, Tsamardinos, triantafillou2015constraint, huang2020causal}. 


\subsection{A Pragmatic approach to validating causal discovery} 
\label{sec:reinterpret}
Drawing inspiration from this application scenario, we provide a pragmatic approach to validate causal discovery methods. We reframe the problem of causal discovery as the prediction of statistical properties of unobserved joint distributions. Specifically, the learning problem reads as follows:
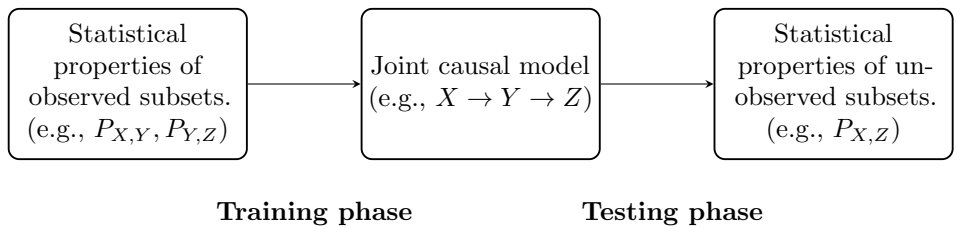
\begin{figure}[h]
	\centering
	\begin{tikzpicture}[
		block/.style={
			draw, 
			fill=white, 
			line width=0.25mm,
			text width=0.23*\columnwidth, 
			minimum height=2cm,
			rounded corners 
		}, 
		font=\small
		]
		\node[block,align=center]
		(step1)
		{Statistical properties of observed subsets. \\ (e.g., $P_{X,Y}, P_{Y,Z}$)};
		\node[right=0.75cm of step1.east,align=center]
		(e1)
		{};
		\node[block,right=1.5cm of step1.east,align=center]
		(step2)
		{Joint causal model \\ (e.g., $X \rightarrow Y \rightarrow Z$)};
		\node[right=0.75cm of step2.east,align=center]
		(e2)
		{};
		\node[block,right=1.5cm of step2.east,align=center]
		(step3)
		{Statistical properties of unobserved subsets. \\ (e.g., $P_{X,Z}$)};
		
		\draw[-stealth]
		(step1.east)--(step2.west)node[pos=0.5, above=0cm]{};
		\draw[-stealth]
		(step2.east)--(step3.west)node[pos=0.5, above=0cm]{};
		
		\node [below=1.3cm of e1] {\textbf {Training phase}};
		\node [below=1.3cm of e2] {\textbf{ Testing phase}};

	\end{tikzpicture}
	\caption{Schema depicting the reframing of causal discovery as a statistical learning problem.}
	\label{schema}
\end{figure}

%
%
Regardless of what kind of statistical properties are meant, under this \textit{statistical paradigm}, causal models entail statements that can be empirically tested without referring to an interventional scenario. Consequently, we drop the ambitious demand of finding `the true' causal model and replace it with a more pragmatic and modest goal of finding causal models that correctly predict unseen joint statistics. 

\begin{remark}
	The joint causal model in Schema \eqref{schema} could be inferred by first inferring marginal causal models (as in integrative causal inference) or directly from the statistical properties of marginal distributions. This distinction is irrelevant for our discussion. 
\end{remark}

\subsection{Why causal models?}
It is not obvious why inferring properties of
unobserved joint distributions from observed ones should take the `detour' via causal models as
visualized in \eqref{schema}.
One could also define a class of {\it statistical} models (that is, a class of joint distributions without any causal interpretation)
that is sufficiently small to yield definite predictions for the desired properties. However, causal models can naturally incorporate causal prior knowledge or causal inductive biases and thereby yield stronger predictions than what may be possible via models without causal semantics. To motiavate this idea, let us consider a simple example. 

\begin{example}[Why causal models are helpful?]
	\label{ex:chain}
	Assume we are given variables $X,Y,Z$ where we observed $P_{X,Y}$ and $P_{Y,Z}$.
	The extension to $P_{X,Y,Z}$ is heavily underdetermined. Now assume that
	we have the additional causal information that $X$ causes $Y$ and $Y$ causes $Z$ (see Figure~\ref{fig:chain}, left), in the sense that both pairs are causally sufficient (see Remark \ref{remark:sufficiency}). In other words, neither
	$X$ and $Y$ nor $Y$ and $Z$ have a common cause.
	This information can be the result of some bivariate causal discovery algorithm that is able to exclude confounding.
	Given that there is, for instance, an additive noise model
	from $Y$ to $Z$  \citep{Kano2003,Hoyer}, a confounder is unlikely
	because it would typically destroy
	the independence of the additive noise term.
	
	\begin{figure}[htp]
	\[
		\hspace{-1em}
		\left. \begin{array}{c}
			\hbox{
			 \resizebox{0.15\textwidth}{!}{%
				\begin{tikzpicture}
					\node[obs] at (0,0) (X) {$X$} ;
					\node[obs] at (2,0) (Y) {$Y$} edge[<-] (X) ;
				\end{tikzpicture}
				}
			}
			\\
			\hbox{
				 \resizebox{0.15\textwidth}{!}{%
				\begin{tikzpicture}
					\node[obs] at (0,0) (Y) {$Y$} ;
					\node[obs] at (2,0) (Z) {$Z$} edge[<-] (Y) ;
				\end{tikzpicture}
				}
			}
		\end{array}
		\right\} \quad \Rightarrow \quad
		\begin{array}{c}
			\hbox{
			 \resizebox{0.25\textwidth}{!}{%
				\begin{tikzpicture}
					\node[obs] at (0,0) (X) {$X$} ;
					\node[obs] at (2,0) (Y) {$Y$} edge[<-] (X) ;
					\node[obs] at (4,0) (Z) {$Z$} edge[<-] (Y) ;
				\end{tikzpicture}
				}
			}
		\end{array}
		\quad \Rightarrow \quad \left\{\begin{array}{c}
			X \independent Z\,| Y \\ P_{X,Y,Z}=P_{X,Y} P_{Z|Y}
		\end{array}\right.
		\]
		\caption{ A simple example where causal information allows to `glue' two distributions to a unique joint distribution.}
		\label{fig:chain}
	\end{figure}
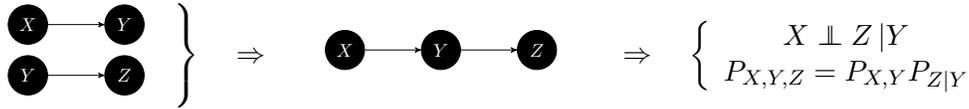
	
	\vspace{0.2cm}
	\noindent
	{\bf Entire causal structure:}
	We can then infer the entire causal structure to be the causal chain
	$X\rightarrow Y\rightarrow Z$ for the following reasons.
	First we show that $X,Y,Z$ is a causally sufficient set of variables:
	A common cause of $X$ and $Z$ would be a common cause of $Y$ and $Z$, too. The pair $(X,Y)$ and $(Y,Z)$ both have no common causes by assumption.
	One checks easily that no DAG with $3$ arrows leaves 
	both pairs unconfounded.
	Checking all DAGs on $X,Y,Z$ with $2$ arrows that have a path  from
	$X$ to $Y$ and from $Y$ to $Z$, we end up with the causal chain
	in Figure~\ref{fig:chain}, middle, as the only option.
	
	\vspace{0.2cm}
	\noindent
	{\bf Resulting joint distribution:}
	This implies
	$
	X \independent Z\,|Y.
	$
	Therefore,
	$
	P_{X,Y,Z}=P_{X,Y} P_{Z|Y}.
	$
\end{example}

Despite it simplicity, Example \ref{ex:chain} demonstrates how incorporating causal prior knowledge into the class of causal models can yield particularly strong predictions about the joint distribution. It is worth noting here that statistical models lack these implications for the joint distribution, as it remains unclear how they can naturally leverage causal prior knowledge to `glue' marginal distributions.

\begin{remark}
	\label{remark:sufficiency}
	Note that we have neglected a subtle issue in discussing Example \ref{ex:chain}.
	There are several different notions of
	what it means that $X$ causes $Y$ in a {\it causally sufficient} way:
	We have above used the purely graphical criterion asking whether there is some variable $Z$ having directed paths to $X$ and $Y$. An alternative option for defining that $X$ influences $Y$ in a causally sufficient way would be to demand that $P_Y^{ do(X=x)}=P_{Y|X=x}$.
	This condition is called `interventional sufficiency' in
	\cite{causality_book}, a condition that
	is testable by interventions on $X$ without referring to a larger background DAG in which $X$ and $Y$ are embedded.
	This condition, however, is weaker than the graphical one and not sufficient for the above argument. This is because one could add the link $X \rightarrow Z$ to the chain $X\rightarrow Y \rightarrow Z$
	and still observe that $P_Z^{do(Y=y)}=P_{Z|Y=y}$, as detailed by Example~9.2 in
	\cite{causality_book}.
	Therefore, we stick to the graphical criterion of causal sufficiency and justify this by the fact that for `generic' parameter values it coincides with interventional sufficiency (which would actually be the more reasonable criterion).
\end{remark}

\paragraph{Causal marginal problem.} The idea that causal constraints can impose meaningful biases on the class of causal models can be supported by a more general motivation of the \textit{causal marginal problem}. Given marginal distributions $P_{S_1},\dots,P_{S_k}$ on sets of variables $S_1, \dots, S_k$,
the problem of existence and uniqueness
of the joint distribution $P_{S_1\cup \cdots \cup S_k}$ that is consistent with the marginals is usually referred to as {\it (probabilistic) marginal problem} \citep{Vorobev1962,Kellerer1964}.
\citet{janzing2018merging}\footnote{Note that the unpublished work \citet{janzing2018merging} is the predecessor of the current work.} introduce the {causal marginal problem} as follows.
Given marginal causal models $M_1,\dots,M_k$ over distinct but overlapping sets of variables $S_1, \dots, S_k$ respectively,
is there a  joint causal model $M$ over $S_1\cup\cdots\cup S_k$ that is \emph{consistent} with the marginal causal models.
A joint causal model $M$ is counterfactually (interventionally) consistent with a marginal model $M_i$, when they agree on all counterfactual (interventional) distributions. Such counterfactual or interventional consistency constraints can impose causal inductive biases over classes of causal models which clearly do not apply for statistical models.
For instance, without formalizing this claim, Example~\ref{ex:chain} suggests that the causal marginal problem may have a unique solution even when the (probabilistic) marginal problem
does not \citep{causalMarginalTalk} -- subject to some genericity assumption explained above.  
\citet{Gresele2022} formally study the special case, where three binary variables $X, Y, Z$ are linked via a V-structure and marginal distributions for $X, Y$ and $Y, Z$ are given.
They show that in this scenario the counterfactual consistency constraints restricts the space of possible joint causal models.
\citet{Guo2023} introduce the term `out-of-Variable (OOV) generalization'  for an agent's ability to handle new situations that involve variables never jointly observed before. As toy example for OOV generalization, they study prediction from distinct yet overlapping causal parents, that is, a scenario where causal directions are known. In contrast, we focus on causal discovery where inferring causal directions is an essential part of the learning task.

With this motivation in mind, we can now summarize the main contributions of this work.

\subsection{Our contributions}
\label{sec:contributions}
The primary contribution of this work is the reinterpretation of the problem of causal discovery as the prediction of joint statistics of unobserved radom variables in the sense described in Section \ref{sec:reinterpret}. This allows us to formalize and study this problem in the framework of statistical learning theory. More explicitly,
\begin{enumerate}[label=(\alph*)]
	
	\item We formalize
	the problem of causal discovery as a standard prediction task where
	the input is a subset (or an ordered tuple) of variables,
	for which we want to test some statistical property.
	The output is a statistical property of that subset (or tuple).
	This way, each observed variable set defines a {\it training} point for inferring the causal model while the unobserved variable sets are the {\it test} instances. (Section~\ref{sec:formal}). 
	
	\item After reinterpreting causal discovery this way, classes of causal models become function classes whose richness can be measured via VC dimension. The problem then becomes directly accessible to 
	statistical learning theory in the following sense.
	Assume we have found a causal model that is consistent with the statistical properties of a large number of observed subsets.
	We can then hope that it also correctly predicts 
	properties of unobserved subsets provided that
	the causal model has been taken from
	a sufficiently `small' class to avoid overfitting on the set of observed statistical properties (see Remark \ref{remark:overfitting}). This `radical empirical' point of view can be
	developed even further: rather than asking 
	whether some statistical property like
	statistical independence is `true', we only ask whether the test at hand rejects or accepts it.\footnote{Asking whether two variables are 'in fact' statistically independent does not make sense for an empirical sample unless the sample
		is thought to be part of an infinite sample -- which is problematic in our finite world.}
	In our reinterpretation we need not ascribe a \enquote{metaphysical} character to statistical independences.
	Hence we can replace the term `statistical
	properties' in the scheme \eqref{schema}
	with `test results'.
	In fact, we do not even have to assume that there is a \enquote{true} underlying causal model.
	We use causal models (a priori) just as predictors of statistical properties. (Section~\ref{sec:VCdim}). 
	
	\item 
	By straightforward
	application of VC learning theory, we then derive error  bounds for the predicted
	statistical properties and discuss how they can be used as guidance for constructing causal hypotheses from not-too-rich classes of hypotheses. (Section \ref{sec:generalization}).
	
	\item We also provide an experimental evaluation of two scenarios with simulated data where we compare prediction errors with our error bounds. (Section \ref{sec:experiments}).
	
	\item Finally, we revisit some conceptual and practical problems of causal discovery and argue that our pragmatic view
	offers a slightly different perspective that is potentially helpful. (Section~\ref{sec:interventions}).
\end{enumerate}

\begin{remark}[\textbf{Benign overfitting}]
	\label{remark:overfitting}
	If there is no label noise, say if all the conditional independences are correctly prescribed in the observed datasets, then models that \textit{overfit} to the observed data (for example, the true DAG) may be superior to the those that do not. Generalization properties for such estimators in the absence of label noise are very well studied \citep{haussler1992decision, bartlett2021failures}. However, in our empirical view point, we can safely omit this discussion. Further note that recent work has shown that the phenomenon of \textit{benign overfitting} -- overfitting to the training data and yet achieving close-to-optimal generalization -- can be observed even in the presence of label noise when learning with overparameterized model classes due to some form of implicit regularization \citep{Bel:2019, belkin2019does, liang2020just, tsigler2020benign, bartlett2020benign, muthukumar2020harmless}. In this paper, we focus on the classical underparameterized regime and refrain from the discussion of generalization in overparameterized systems.
\end{remark}


\subsection{Related work}
The field of integrative causal inference \citep{Tsamardinos} is the work that is closest to the present paper. \citet{Tsamardinos} provides algorithms that use causal inference to combine knowledge from heterogenous data sets and to predict conditional independences between variables that have not been jointly observed. In contrast, our main contribution is the conceptual task of framing causal discovery as a predictive problem which allows causal models to be empirically testable in the iid scenario thereby setting the scene for learning theory. Furthermore, in contrast to \cite{Tsamardinos}, the term `statistical properties' in our setting need not necessarily
refer to conditional independences. There is a broad variety of 
new approaches that infer causal directions 
from statistical properties other than conditional independences \cite{Kano2003,SunLauderdale,Hoyer,Zhang_UAI,
	deterministic,SecondOrder,
	stegle2010probabilistic,peters2010identifying,Mooij2016,Marx2017}. On the other hand, the causal model inferred from the observations may entail statistical properties other than conditional independences -- subject to the model assumptions
on which the above-mentioned inference procedures rely.

\section{The formal setting \label{sec:formal}}

Below we will usually refer to some given set of variables $S:=\{X_{1},\dots,X_{n}\}$ whose subsets are considered. Whenever this cannot cause any confusion, we will not carefully distinguish between the {\it set} $S$ and the {\it vector} $\bX:=(X_{1},\dots,X_{n})$ and also use the term
'joint distribution $P_S$' although the order of variables certainly matters.

\subsection{Statistical properties}
Statistical properties are the crucial concept of this work. On the one hand, they are used to infer causal structure. On the other hand, causal structure is used to predict statistical properties.
\begin{definition}[statistical property]
	Let $S = \{X_1, \dots, X_n\}$ a set of variables.
	A 
	statistical property $Q$ with range $\cY$ is
	given by a function 
	\[
	Q: \Delta_S \rightarrow \cY
	\]
	where $\Delta_S$ denotes the space of joint distributions of $k$-tuples of variables $Y_1, \dots, Y_k\in S$ and $\cY$ denotes some output 
	space. Often we will consider
	binary or real-valued properties, that is
	$\cY=\{0,1\}$, $\cY=\{-1,+1\}$, or $\cY=\R$, respectively. 
\end{definition}  
By slightly abusing terminology, the term `statistical property' will sometimes refer
to the value in $\cY$ that is the output of
$Q$ or to the function $Q$ itself. This will, hopefully, cause no confusion. 

Here, $Q$ may be defined for fixed size
$k$ or for general $k$. Moreover, we will consider properties that depend on the ordering of the variables $Y_1,\dots,Y_k$, those that do not depend on it, or those that are invariant under some permutations $k$ variables. This will be clear from the context.
We will refer to $k$ tuples for which
part of the order matters as `partly ordered 
tuples'.
To given an impression about the variety of 
statistical properties we conclude the section with a list of examples.

We start with an example for a binary property that
does not refer to an ordering:
\begin{example}[statistical independence]
	\[
	Q(P_{Y_1,\dots,Y_k}) = \left\{ \begin{array}{cc}
		1 & \hbox{ for $Y_j$ jointly independent } \\
		0 &  \hbox{ otherwise } 
	\end{array}\right.
	\]
\end{example}
The following binary property 
allows for some permutations of variables:
\begin{example}[conditional independence or partial uncorrelatedness]\label{ex:ci}
	\[
	Q(P_{Y_1,\dots,Y_k}) = \left\{ \begin{array}{ccc}
		1 & \hbox{ for } &  Y_1 \independent Y_2 \,| Y_3,\dots,Y_k \\
		0 &  \hbox{ otherwise } &
	\end{array}\right.
	\]
	Likewise, $Q(P_{Y_1,\dots,Y_k})$ could indicate whether 
	$Y_1$ and $Y_2$ have zero partial correlations, given $Y_3,\dots,Y_k$ (that is, whether they are uncorrelated after linear regression on $Y_3,\dots,Y_k$).
\end{example}

To emphasize that our causal models are not only used to predict conditional independences but also other statistical properties we also mention linear additive noise models \citep{Kano2003}:
\begin{example}[existence of linear additive noise models]
	\label{ex:lingam}
	$Q(P_{Y_1,\dots,Y_k})=1$ if and only
	if there is a matrix $A$ with entries $A_{ij}$,
	that is lower triangular, 
	such that 
	\begin{equation}\label{eq:lingam}
		Y_i = \sum_{j<i} A_{ij} Y_j +N_j,
	\end{equation}
	where $N_1,\dots,N_k$ are jointly independent noise variables. If no such additive linear model exists, we
	set 
	$Q(P_{Y_1,\dots,Y_k})=0$. 
\end{example}
Lower triangularity means that there is a DAG
such that $A$ has non-zero entries $A_{ij}$ whenever there is an arrow from $j$ to $i$.
Here, the entire order of variables matters.
Then \eqref{eq:lingam} is a linear structural equation.
Whenever the noise variables $N_j$ are non-Gaussian, linear additive noise models allow for 
the unique identification of 
the causal DAG 
\citep{Kano2003} if one assumes that the true generating process has been linear.
Then, $Q(P_{Y_1,\dots,Y_k})=1$ holds for those orderings of variables that are compatible with the true DAG.
This way, we have a statistical property that is directly linked to the causal structure (subject to a strong assumption, of course).

The following simple binary property will also
play a role later:
\begin{example}[sign of correlations]\label{ex:signcorrelations}
	Whether a pair of random variables is positively or negatively correlated defines a simple binary
	property in a scenario where all variables are correlated:
	\[
	Q(P_{Y_1,Y_2}) = \left\{\begin{array}{cl} 1 & \hbox{  if } \cov(Y_1,Y_2) >0\\
		-1 & \hbox{ if } \cov(Y_1,Y_2)<0 
	\end{array}\right.
	\]
\end{example}
Note that positivity of the covariance matrix already restricts $Q$, but we will later see a causal model class that   restricts $Q$ even further beyond this constraint.

Finally, we mention a statistical property that is not binary but positive-semidefinite matrix-valued:
\begin{example}[covariances and correlations]\label{ex:cov}
	For $k$ variables $Y_1,\dots,Y_k$
	let $\cY$ be the set of positive semi-definite matrices. Then define
	\[
	Q: P_{Y_1,\dots,Y_k} \mapsto \Sigma_{Y_1,\dots,Y_k},
	\]
	where $\Sigma_{Y_1,\dots,Y_n}$ denotes the
	joint covariance matrix of $Y_1,\dots,Y_n$.
	For $k=2$, one can also get a real-valued property by focusing on the off-diagonal term.
	One may then define a map $Q$  by
	\[
	Q(P_{Y_1,Y_2}) := \cov(Y_1,Y_2),
	\]
	or alternatively, if one prefers correlations, define
	\[
	Q(P_{Y_1,Y_2}) := \corr(Y_1,Y_2).
	\]
\end{example}

\subsection{Statistical and causal models}

The idea of this paper is that causal models are used to
predict statistical properties, but a priori, the models need not be causal.
One can use Bayesian networks, for instance, to encode conditional statistical independences with or without interpreting the arrows as formalizing causal influence.
For the formalism introduced in this section it does not matter whether one interprets the models as causal or not.
Example~\ref{ex:chain}, however, suggested
that model classes that come with a causal semantics are particularly intuitive 
regarding the statistical properties they predict.
We now introduce our notion of `models':
\begin{definition}[models for a statistical property]
	\label{def:model}
	Given a set $S:=\{X_1,\dots,X_n\}$ of variables 
	and some statistical property $Q$, 
	a model $M$ for $Q$ 
	is a class of joint distributions 
	$P_{X_1,\dots,X_n}$ that coincide regarding the output of $Q$, that is,
	\[
	Q(P_{Y_1,\dots,Y_k}) = Q(P'_{Y_1,\dots,Y_k}) \quad \forall P_{Y_1,\dots,Y_k},P'_{Y_1,\dots,Y_k} \in M,
	\]
	where\footnote{To avoid threefold indices we use $Y_j$ instead of $X_{i_j}$ here.} $Y_1,\dots,Y_k \in S$. 
	Accordingly, the property 
	$Q_M$  predicted by the model $M$ is given by a function
	\[
	(Y_1,\dots,Y_k) \mapsto Q_M\left[(Y_1,\dots,Y_k)\right] := Q(P_{Y_1,\dots,Y_k}),
	\]
	for all $P_{X_1,\dots,X_n}$ in $M$,
	where $(Y_1,\dots,Y_k)$ runs over all
	allowed input (partly ordered) tuples of $Q$.
\end{definition}
Formally, the `partly ordered tuples' are
equivalence classes  in $S^k$, where equivalence 
corresponds to irrelevant reorderings of the tuple. To avoid cumbersome formalism, we will
just refer to these equivalence classes as `the allowed inputs'.

Later, such a model will be, for instance, a DAG $G$ and the property $Q$ formalizes all conditional independences that hold for the respective Markov equivalence class. 
To understand the above terminology, note that $Q$
receives a distribution as input and the output of $Q$ tells us the respective property
of the distribution (e.g. whether independence holds). In contrast, $Q_M$ receives a set of nodes (variables) of the DAG as inputs and 
tells us the property entailed by $M$.
The goal will be to find a model $M$ for which
$Q_M$ and $Q$ coincide for the majority of observed tuples of variables.

We now describe a few examples of causal models as predictors of statistical properties that we have in mind. 
Our most prominent one reads:
\begin{example}[DAG as model for conditional independences]\label{ex:dagsci}
	Let $G$ be a DAG with nodes $S:=\{X_1,\dots,X_n\}$ and $Q$ be the set of conditional independences as in Example~\ref{ex:ci}.
	Then, let $Q_G$ be the function on
	$k$-tuples from $S$ defined by
	\[
	Q_G\left[(Y_1,\dots,Y_k)\right] :=0
	\]
	if and only if the Markov condition implies $Y_1\independent Y_2\,|Y_3\dots,Y_k$, and
	\[
	Q_G\left[ (Y_1,\dots,Y_k)\right] :=1
	\] 
	otherwise. 
\end{example}
Note that $Q_G(.)=1$ does not mean that the
Markov condition implies dependence, it only
says that it does not imply independence.
However, if we think of $G$ as a causal DAG, the common assumption of causal faithfulness \citep{Spirtes1993} states that all dependences
that are allowed by the Markov condition occur in reality. Adopting this assumption, 
we will therefore interpret $Q_G$ as a function that predicts dependence or independence, instead of making no prediction if the Markov condition allows dependence.

We also mention a particularly simple
class of DAGs that will appear as an interesting example later: 
\begin{example}[DAGs consisting of a single colliderfree path]\label{ex:cfreepath}
	Let $\cG$ be the set of DAGs that consist of a single colliderfree path 
	\[
	X_{\pi(1)} - X_{\pi(2)} - X_{\pi(3)} - \cdots
	- X_{\pi(n)}, 
	\]
	where the directions of the arrows are such that
	there is no variable with two arrowheads.
	Colliderfree paths have the important property that any dependence between two non-adjacent nodes is screened off by any variable that lies between the two nodes, that is, 
	\[
	X_j \independent X_k|\, X_l,
	\]
	whenever $X_l$ lies between $X_j$ and $X_k$. 
	If one assumes, 
	in addition, that
	the joint distribution is Gaussian,
	the partial correlation
	between $X_j$ and $X_k$, given $X_l$,
	vanishes. This implies that 
	the correlation coefficient of any two nodes is given by the product of pairwise correlations along the path:
	\begin{equation}\label{eq:corrproduct}
		\corr (X_j,X_k)
		= \prod_{i=\pi^{-1}(j)}^{\pi^{-1}(k)-1} \corr (X_{\pi(i)},X_{\pi(i+1)}) =: \prod_{i=\pi^{-1}(j)}^{\pi^{-1}(k)-1} r_i. 
	\end{equation}
	This follows easily by induction because $\corr(X,Z)=\corr (X,Y)\corr(Y,Z)$ for any three variables $X,Y,Z$ with $X\independent Z\,|Y$. 
	Therefore, such a DAG, together with all the correlations between adjacent nodes, 
	predicts all pairwise correlations. 
	We therefore specify our model by 
	$M:=(\pi,r)$, that is, the ordering of nodes and correlations of adjacent nodes. 
\end{example}

The following example shows that a DAG can entail also properties that are more sophisticated than just conditional independences and correlations: 
\begin{example}[DAGs and linear non-Gaussian additive noise]
	\label{ex:daglingam}
	Let $G$ be a DAG with nodes $S:=\{X_1,\dots,X_n\}$ and $Q$ be the linear additive noise property in Example~\ref{ex:lingam}.
	Let $Q_G$ be the function
	on
	$k$-tuples from $S$ defined by
	\[
	Q_G((Y_1,\dots,Y_k)) :=1
	\]
	if and only if the following two conditions hold:\\ 
	(1) $Y_1,\dots,Y_k$ is a
	causally sufficient subset from $S$ in $G$ and, that is, no two different $Y_i,Y_j$ have a common ancestor in $G$ \\
	(2) the ordering $Y_1,\dots,Y_k$ is consistent with $G$, that is,
	$Y_j$ is not ancestor of $Y_i$ in $G$
	for any $i<j$.    
\end{example}
In contrast to $Q$ from Example~\ref{ex:lingam}, $Q_G$ predicts from the graphical structure whether the joint distribution of some subset of variables admits a linear additive noise model. The idea is the following. Assuming that the entire joint distribution of all $n$ variables has been generated by a linear additive noise model \citep{Kano2003}, any $k$-tuple 
$(Y_1,\dots,Y_k)$ also admits a linear additive noise model provided that (1) and (2) hold. 
This is because marginalizations of linear additive noise models remain linear additive noise models whenever one does not marginalize over common ancestors.\footnote{Note that the class of {\it non-linear} additive noise models \citep{Hoyer} is not closed under marginalization.}  
Hence, conditions (1) and (2) are clearly sufficient.
For generic parameter values of the underlying linear model the two conditions are also necessary because 
linear non-Gaussian models render
causal directions uniquely identifiable and also admit the detection of hidden common causes \citep{HoyerLatent08}.

\subsection{Testing properties on data}

So far we have introduced statistical properties as mathematical properties of distributions.
In real-world applications, however, we want to
predict the outcome of a test on empirical data.
The task is no longer to predict whether some set of variables is `really' conditionally independent, we just want to predict whether the statistical test at hand accepts independence.
Whether or not the test is appropriate for the respective mathematical property $Q$ is not relevant for the generalization bounds derived later. If one infers DAGs, for instance, by 
partial correlations and uses these DAGs only to infer partial correlations, it does not matter that non-linear relations actually prohibit to replace 
conditional independences with partial correlations. 
The reader may get confused by these remarks
because now there seems to be no requirement 
on the tests at all if it is not supposed to be a good test for the mathematical property $Q$.
This is a difficult question. One can say, however, that for a test that is entirely
unrelated to some property $Q$ we have no guidance what outcomes of our test a causal hypothesis should predict. The fact that partial correlations, despite all their limitations, approximate conditional independence, 
does provide some justification for expecting 
vanishing partial correlations in many cases where there is d-separation in the causal DAG.

We first specify the information provided by a data set.
\begin{definition}[data set]
	Each data set $D_j$ is an $l_j\times k_j$ matrix of observations, where $l_j$ denotes the sample size and $k_j$ the number of variables. 
	Further, the dataset contains a $k_j$-tuple
	of values from $\{1,\dots,n\}$ 
	specifying the $k_j$ variables $Y_1,\dots,Y_{k_j} \subset \{X_1,\dots,X_n\}$ the samples refer to. 
\end{definition}
To check whether the variables under consideration in fact satisfy the property predicted by the model
we need some statistical test (in the case of binary properties) or an estimator (in the case of real-valued or other properties). 
Let us say that we are given 
some test or estimator for a property $Q$, formally defined as follows:
\begin{definition}[statistical test / estimator for $Q$]
	A test (respective estimator for non-binary properties) for the statistical property $Q$ 
	with range $\cY$ is a map
	\[
	Q_T: D \mapsto Q_T(D) \in \cY, 
	\]
	where $D$ is a data set that involves 
	the observed instances of
	$Y_1,\dots,Y_k$, where $(Y_1,\dots,Y_k)$ is
	a partly ordered tuple that defines
	an allowed input of $Q$.
	$Q_T(D)$ is thought to indicate the outcome
	of the test or the estimated value, respectively. 
\end{definition}

\subsection{Phrasing the task as standard prediction problem}

Our learning problem now reads: given the data sets
$D_1,\dots,D_l$ with the $k$-tuples $S_1,\dots,S_l$
of variables, find a model $M$ such that
$Q_M(S_j) = Q_T(D_j)$ for all data sets $j=1,\dots,l$ or, less demanding, for most of the data sets.
However, more importantly, we would like to choose $M$ such that $Q_M(S_{l+1}) =Q_T(D_{l+1})$ will most probably also hold for
a {\it future} data set $D_{l+1}$.  

The problem of constructing a causal model now becomes a standard learning problem
where the training as well as the test examples  
are {\it data sets}. Note that also 
\citet{Lopez2015} phrased a causal discovery problem as standard learning problem. There,
the task was to classify two variables as `cause' and `effect' after getting
a large number of cause-effect pairs as training examples. Here, however, the data sets refer to observations from different subsets of variables that are assumed to follow a joint distribution
over the union of all variables occurring in any of the data sets.

Having phrased our problem as a standard prediction scenario whose inputs are subsets of variables, we now introduce the usual notion of empirical error on the training data accordingly:
\begin{definition}[empirical error]
	Let $Q$ be a statistical property,
	$Q_T$ a statistical test, and
	$D:=\{D_1,\dots,D_k\}$ a collection of data sets referring to the variable tuples $S_1,\dots,S_k$. Then the empirical training error of model $M$ is defined by
	\[
	L(M) := \frac{1}{k} \sum_{j=1}^k |Q_T (D_j) - Q_M(S_{D_j})|. 
	\]
\end{definition}

Note that our theory does not prefer one model over another if they agree in terms of the predictions they make.
For example, with the independence property as defined in \cref{ex:dagsci}, all DAGs in the same Markov equivalence class are also equivalent with respect to the empirical error $L(M)$.
This further emphasizes that we are not necessarily looking for a \emph{true} model.

To see why this paradigm change can be helpful, assume there was a ground truth model $M$ for which the test $Q_T$ correctly outputs the statistical properties.
Then $M$ will certainly be one of the optimal models.
Yet, if we are bound to make errors (either in evaluating the property via $Q_T$ or in estimating $M$), it becomes less obvious which model to pick.
Graphical metrics like structural Hamming distance (SHD) \citep{tsamardinos2006max} are used frequently to benchmark causal discovery, although it is still an open debate how to quantify the quality of causal models \citep{gentzel2019case}.
In our framework the focus is shifted.
Here, causal models are seen as predictors of statistical properties.
Therefore the best model is the one that predicts the statistical properties the most accurately.
In this sense, we do not even need to reference a \enquote{true} model.
Consider the following example.
\begin{example}[Ground truth vs. prediction]
	\label{ex:gt_vs_pred}
	Assume there is a true causal model for the variables $X, Y, Z$, such that $X$ causes $Y$ and $Z$ is a confounder, as visualized in \cref{fig:confounder_shd_a}.
	Further assume that the confounding effect of $Z$ is very weak, to the extent that the statistical independence test $Q_T$ outputs $X\ind Z$ and $Z\ind Y$.
	Let $\hat{G}_1$ be a graph that reflects these independences as shown in \cref{fig:confounder_shd_b}.
	In the sense of our framework, $\hat{G}_1$ is a good predictor of $Q_T$, as for every tuple of variables (i.e. dataset) it correctly predicts the output of the  independence tests.
	Now consider $\hat{G}_2$ in \cref{fig:confounder_shd_c}.
	$\hat{G}_2$ is closer to the ground truth $G$ with respect to SHD.
	Yet, it reflects the observed independences quite poorly.
	So if we are interested in the output of the independence tests and data at hand, $\hat{G}_1$ is more \enquote{useful} than $\hat{G}_2$.
	Although, we do not claim that $\hat{G}_1$ is \emph{generally} better or worse than $\hat G_2$, and
	the statement is to be understood with respect to the statistical tests considered.
	
	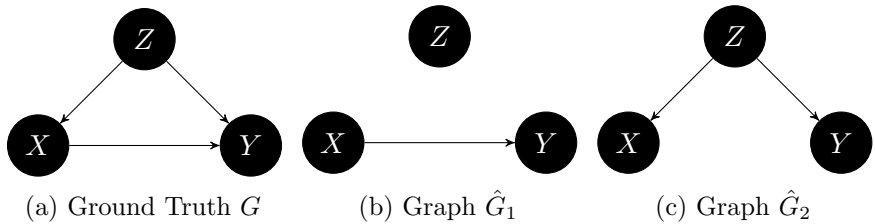
\begin{figure}[htp]
		\centering
		\begin{subfigure}[b]{.3\textwidth}
			\centering
			\begin{tikzpicture}
				\node[obs]  (X) {$X$};
				\node[obs, above right of=X]  (Z) {$Z$};
				\draw (Z) edge[->] node[above]{} (X) ;
				\node[obs, below right of=Z]  (Y) {$Y$} ;
				\draw (X) edge[->] node[below]{} (Y);
				\draw (Z) edge[->] node[above]{}(Y);
			\end{tikzpicture}
			\caption{Ground Truth $G$}
			\label{fig:confounder_shd_a}
		\end{subfigure}
		\begin{subfigure}[b]{.3\textwidth}
			\centering
			\begin{tikzpicture}
				\node[obs]  (X) {$X$};
				\node[obs, above right of=X]  (Z) {$Z$};
				\node[obs, below right of=Z]  (Y) {$Y$} ;
				\draw (X) edge[->] (Y);
			\end{tikzpicture}
			\caption{Graph $\hat G_1$}
			\label{fig:confounder_shd_b}
		\end{subfigure}
		\begin{subfigure}[b]{.3\textwidth}
			\centering
			\begin{tikzpicture}
				\node[obs]  (X) {$X$};
				\node[obs, above right of=X]  (Z) {$Z$};
				\node[obs, below right of=Z]  (Y) {$Y$} ;
				\draw (Z) edge[->] (Y);
				\draw (Z) edge[->] (X);
			\end{tikzpicture}
			\caption{Graph $\hat{G}_2$}
			\label{fig:confounder_shd_c}
		\end{subfigure}
		\caption{ Let $G$ be the true model. If $Q_T$ falsely outputs $X\ind Z$ and $Z\ind Y$ then $\hat{G}_1$ is a better predictor of $Q_T$ than $\hat{G}_2$, even though $\hat{G}_2$ is graphically closer to the ground truth $G$.}
		\label{fig:confounder_shd}
	\end{figure}
\end{example}

Finding a model $M$ for which the training error
is small does not guarantee, however, that 
the error will also be small for future test data. If $M$ has been chosen from a `too rich'
class of models, the small training error may be 
a result of overfitting. Fortunately we have phrased our learning problem in a way that
the richness of a class of causal models
can be quantified by standard concepts from statistical learning theory. 
This will be discussed in the following section.

\section{Capacity of classes of causal models \label{sec:VCdim}}

We have formally phrased our problem
as  a prediction problem where the task is to predict the outcome in $\cY$ of $Q_T$ for
some test $T$ applied to an unobserved variable set. 
We now assume that we are given a class of models $\cM$ defining 
statistical properties $(Q_M)_{M\in \cM}$ that are supposed to predict the outcomes of $Q_T$.

\subsection{Binary properties}
Given some binary statistical property, we can straightforwardly apply the notion of \mbox{VC dimension} \citep{Vapnik}  to classes $\cM$ and define:
\begin{definition}[VC dimension of a model class for binary properties]
	Let $S$ be a
	set of variables $X_1,\dots,X_n$ and $Q$ be a binary property. Let
	$\cM$ be a class of models for $Q$, that is, each $M\in \cM$ defines a map
	\[
	Q_M: (Y_1,\dots,Y_k) \mapsto Q_M\left[(Y_1,\dots,Y_k)\right] \in \{0,1\}. 
	\]
	Then the VC dimension of $\cM$ is the largest number $h$ such that there
	are $h$ allowed inputs $S_1,\dots,S_h$
	for $Q_M$ such that the
	restriction of all $M\in \cM$ to $S_1,\dots,S_h$
	runs over all $2^h$ possible binary functions. 
\end{definition}
Since our model classes are thought to be given by causal hypotheses the following class  is our most important example although we will later further restrict the class to get stronger generalization bounds:
\begin{lemma}[VC dimension of 
	conditional independences entailed by DAGs] 
	\label{lem:dags}
	Let $\cG$ be the set of DAGs with nodes
	$X_1,\dots,X_n$. For every $G\in \cG$,
	we define $Q_G$ as in Example~\ref{ex:dagsci}.
	Then the VC dimension $h$ of
	$(Q_G)_{G\in \cG}$ satisfies
	\begin{equation}\label{eq:VCDAGs}
		h\leq n \log_2 n + n(n-1)/2\in O(n^2).
	\end{equation}
\end{lemma}
\begin{proof}
	The number $N_n$ of DAGs on $n$ labeled nodes can easily be upper bounded by the number of orderings times the number of choices to draw an edge or not.
	This yields
	$N_n < n! 2^{n(n-1)/2}$.
	Using Stirling's formula we obtain
	\[
	n! < e^{1/(12n)} \sqrt{2\pi n} \left(\frac{n}{e}\right)^n < n^n, 
	\] 
	and thus $N_n< n^n 2^{n(n-1)/2}$.
	Since the VC dimension of a class cannot be larger than the binary logarithm of the number of
	elements it contains, \eqref{eq:VCDAGs} easily follows.
\end{proof}
It would actually be desirable to find a bound for the VC dimension that is smaller than the logarithm of the number of classifiers, since that would be
a more powerful application of VC theory. We leave this to future work.

\vspace{0.3cm}
Note that the number of possible conditional independence tests of the form $Y_1\independent Y_2\,|Y_3$ already grows faster than the VC dimension, namely with the third power.\footnote{The question for which patterns of conditional dependences and independences there exists a joint distribution is in general hard to answer. Conditional independences satisfy semi-graphoid axioms \citep{Lauritzen} and therefore entail further conditional independences. Thus, the set of joint distributions already define restricted class of predictors of conditional independences (although 
the outcomes of empirical tests need not respect these restrictions). However, we will not further elaborate on this since even the class of all DAGs seems to large for our purpose.} 
Therefore, the class of all DAGs already defines a restriction since it is not able to explain 
all possible patterns of conditional (in)dependences even 
when one conditions on one variable only.

Nevertheless, the set of all DAGs may be too large for the number of data sets at hand. We therefore mention the following more restrictive class given by so-called polytrees, that is, DAGs whose skeleton is a tree (hence they contain no undirected cycles). 
\begin{lemma}[VC dimension of 
	cond.~independences entailed by
	polytrees]\label{lem:polytrees}
	Let $\cG$ be the set of polyntrees with nodes
	$X_1,\dots,X_n$. For every $G\in \cG$,
	we define $Q_G$ as in Example~\ref{ex:dagsci}.
	Then the VC dimension $h$ of
	$(Q_G)_{G\in \cG}$ satisfies
	\begin{equation}\label{eq:polytrees}
		h\leq n (\log_2 n +1). 
	\end{equation}
\end{lemma}
\begin{proof}
	According to Cayley's formula, the number of trees with $n$ nodes reads
	$n^{n-2}$ \citep{Aigner1998}. The number of Markov equivalence classes of polytrees can be bounded from above by 
	$
	2^{n-1} -n +1
	$
	\citep{Radhakrishnan2017}.
	Thus the number of Markov equivalence classes
	of polytrees is upper bounded by
	\begin{equation}\label{eq:numberOffunc}
		n^{n-2} (2^{n-1} -n +1)\leq n^{n-2} 2^n.
	\end{equation}
	Again, the bound follows by taking the logarithm.
\end{proof}

\vspace{0.3cm}

We will later use the following result:
\begin{lemma}[VC dimension of sign of correlations along a path]\label{Lem:VCcorrPath}
	Consider the set of DAGs on $X_1,\dots,X_n$ that consist of a single
	colliderfree path as in Example~\ref{ex:cfreepath} and assume multivariate Gaussianity. The sign of pairwise 
	correlations is then determined by
	the permutation $\pi$ that aligns the graph
	and the sign of correlations of all adjacent pairs. We thus parameterize a model by $M:=(\pi,s)$ where the vector $s:=(s_1,\dots,s_n)$ denotes the signs of adjacent nodes. 
	The full model class $\cM$ is obtained when $\pi$ runs over the entire group of permutations and $s$ over all combinations in $\{-1, +1\}^n$. 
	Let $Q$ be the
	property indicating the sign of the correlation of any two variables as in
	Example~\ref{ex:signcorrelations}. 
	Then the VC dimension of $(Q_M)_{M\in \cM}$ is at most $n$.
\end{lemma}
\begin{proof}
	Defining 
	\[
	s_j := \prod_{i=1}^{\pi^{-1}(j)-1} {\rm sign}( \corr (X_{\pi(i)},X_{\pi(i+1)}))
	\]
	we obtain 
	\[
	{\rm sign}( \corr(X_i,X_j) ) = s_i s_j,
	\]
	due to \eqref{eq:corrproduct}.
	Therefore, the signs of all pairwise correlations can be computed from
	$s_1,\dots,s_n$. Since there are $2^n$ possible assignments for these values, $\cG$ thus induces $2^n$ functions and thus the VC dimension is at most $n$.
\end{proof}

\subsection{Real-valued statistical properties}

We also want to obtain quantitative statements about the strength of dependences and therefore
consider also the correlation as an example of a real-valued property.
\begin{lemma}[correlations along a path]
	Let $\cM$ be the model class whose elements
	$M$ are colliderfree paths together with a list of all correlations of adjacent pairs of nodes, see
	Example~\ref{ex:cfreepath}. Assuming also multi-variate Gaussianity, $M$, again, defines
	all pairwise correlations and we 
	can thus define the model induced property
	\[
	Q_M\left[(X_j,X_k)\right]:= \corr_M (X_j,X_k),
	\]
	where the term on the right hand side 
	denotes the correlation determined  by the model $M:=(\pi,r)$ as introduced in Example~\ref{ex:cfreepath}.  
	Then the VC dimension of $(Q_M)_{M \in \cM}$ is in $O(n)$.
\end{lemma}
\begin{proof} We assume, for simplicity, that
	all correlations are non-zero. 
	To specify the absolute value of the correlation between
	adjacent nodes we define the parameters
	\[
	\beta_{i} := \log |\corr_M  (X_{\pi(i-1)},X_{\pi(i)})|.
	\]
	To specify the sign of those correlations we define the binary values
	\[
	{\rm g}_{i} :=     \left\{\begin{array}{cc}
		1 & \hbox{ for }  \corr_M (X_{\pi(i-1)},X_{\pi(i)})  < 0  \\
		0 & \hbox{ otherwise } \end{array}\right.,
	\]
	for all $i\geq 2$. 
	
	It will be convenient to introduce the parameters
	\[
	\alpha_j := \sum_{i =2}^{j} \beta_i,
	\]
	which are cumulative versions of the `adjacent log correlations' $\beta_i$.
	Likewise, we introduce the binaries
	\[
	s_j := \left(\sum_{i =2}^{j} g_i\right) {\rm mod }\, 2,
	\]
	which indicate whether the number of negative correlations along the chain from its beginning 
	is odd or even. 
	
	This way, the correlations between any two nodes can be computed from $\alpha$ and $s$:
	\[
	\corr_M (X_j,X_k) = (-1)^{s_{\pi^{-1}(j)}+s_{\pi^{-1}(k)}} \, e^{|\alpha_{\pi^{-1}(j)} -\alpha_{\pi^{-1}(k)}|}. 
	\]
	For technical reasons we define $\corr$ formally as a 
	function of {\it ordered} pairs of variables
	although it is actually symmetric in $j$ and $k$. 
	We are interested in the VC dimension of the family  $F:=(f_M)_{M\in \cM}$ 
	of real-valued functions defined by
	\[
	f_{M}(j,k):= \corr_{M} (X_j,X_k)=:\rho^M_{i,j}.
	\]
	Its VC dimension is defined as the VC dimension of the
	set of classifiers $C:=(c^\gamma_{M})_{M,\gamma}$ with
	\[
	c^\gamma_{M} (j,k) :=   \left\{\begin{array}{cc}
		1 & \hbox{ for }  \rho^M_{j,k}  \geq  \gamma  \\
		0 & \hbox{ otherwise } \end{array}\right.,
	\]
	To estimate the VC dimension of $C$ we 
	compose it from classifiers whose VC dimension is easier to estimate. 
	
	We first define the family of classifiers given by
	$C^>:=(c^{>\theta}_\alpha)_{\alpha\in \R^-,\theta\in \R}$ with
	\[
	c^{>\theta}_\alpha (j,k):=   \left\{\begin{array}{cc}
		1 & \hbox{ for }  \alpha_{\pi^{-1}(j)}-\alpha_{\pi^{-1}(k)} \geq  \theta  \\
		0 & \hbox{ otherwise } \end{array}\right..
	\] 
	Likewise, we define
	$C^<:=(c^{<\theta}_\alpha)_{\alpha\in \R^-,\theta\in \R}$ with
	\[
	c^{<\theta}_\alpha (j,k):=   \left\{\begin{array}{cc}
		1 & \hbox{ for }  \alpha_{\pi^{-1}(j)}-\alpha_{\pi^{-1}(k)} <  \theta  \\
		0 & \hbox{ otherwise } \end{array}\right..
	\] 
	The VC dimensions pf $C^>$ and $C^<$ are at most $n+1$ because they are given by linear functions on the space of all possible $\alpha \in \R^n$ \citep{Vapnik1995}, Section 3.6, Example 1.  
	Further, we define a set of classifiers that classify only according to the sign of the correlations:
	\[
	S:= (c^M_+) \cup (c^M_-), 
	\]
	where 
	\[
	c^M_+(j,k) := \left\{ \begin{array}{cc} 1 & \hbox{ if } \rho^M_{j,k} \geq 0 \\
		0 & \hbox{ otherwise } \end{array}\right..
	\]
	Likewise, we set
	\[
	c^M_-(j,k) := \left\{ \begin{array}{cc} 1 & \hbox{ if } \rho^M_{j,k} < 0 \\
		0 & \hbox{ otherwise } \end{array}\right..
	\]
	Since both components of $S$ have VC dimension $n$ at
	most, the VC dimension of $S$ is in $O(n)$.

	For $\gamma> 0$, $\rho^M_{j,k} \geq \gamma$ is equivalent to
	\[
	(\rho^M_{j,k} \geq 0) \wedge (\alpha_{\pi^{-1}(j)}-\alpha_{\pi^{-1}(k)} \geq \log \gamma)  \wedge
	(\alpha_{\pi^{-1}(k)} -\alpha_{\pi^{-1}(j)} \geq \log \gamma).   
	\]
	Therefore, 
	\[
	c^\gamma_M \in S  \sqcap C^> \sqcap C^<,
	\] 
	for all $\gamma >0$, where  $\sqcap$ denotes the intersection of
	`concept classes' \citep{vanDerVaart2009}  given by 
	\[
	C_1 \sqcap C_2:= (c_1 \cap c_2)_{c_1\in C_1,c_2 \in C_2}.
	\]
	Likewise, the union of concept classes is given by
	\[
	C_1 \sqcup C_2:= (c_1 \cup c_2)_{c_1\in C_1,c_2 \in C_2},
	\]
	as opposed to the set-theoretic unions and intersections.
	
	For $\gamma <0$, $\rho^M_{j,k}\geq \gamma$ is equivalent to
	\[
	(\rho^M_{j,k} \geq 0) \vee \left\{ (a_{\pi^{-1}(j)} - \alpha_{\pi^{-1}(k)} \geq \log |\gamma |) \wedge (\alpha_{\pi^{-1}(k)} -\alpha_{\pi^{-1}(j)} \geq \log |\gamma|)\right\}.
	\]
	Hence, 
	\[
	c^\gamma_M \in S  \sqcup [ C^> \sqcap C^<],
	\]
	for all $\gamma<0$.
	We then obtain:
	\[
	C \subset (S \sqcap C^> \sqcap C^<) \cup (S \sqcup [ C^> \sqcap C^<]).
	\]
	Hence, $C$ is a finite union and intersection of concept classes and set theoretic union,
	each having VC dimension in $O(n)$. 
	Therefore, $C$ has VC dimension in $O(n)$ \citep{vanDerVaart2009}.
\end{proof}

\section{Generalization bounds \label{sec:generalization}}

\subsection{Binary properties\label{subsec:VCboundsbinary}}

After we have seen that in our scenario causal models
like DAGs define classifiers
in the sense of standard learning scenarios,
we can use the usual VC bounds like  Theorem~6.7 in \cite{Vapnik06} to guarantee generalization to future data sets.
To this end, we need to assume that the data sets are sampled from some distribution of data sets, an assumption that will be discussed at the end of this section.
\begin{theorem}[VC generalization bound]
	\label{thm:VCbound}
	Let $Q_T$ be a statistical test for some statistical binary property and
	$\cM$ 
	be a model class with VC dimension $h$
	defining some model-induced property $Q_M$.  
	Given $k$ data sets $D_1,\dots,D_k$ sampled according to a distribution $P_D$. Then
	\begin{equation}\label{eq:VCbinarygeneral}
		\Exp\left[|Q_T(D)-Q_M(D)|\right] \leq    \frac{1}{k} \sum_{j=1}^k |Q_T(D_j)-Q_M(S_{D_j})|
		+ 2\sqrt{\frac{h\left(\ln \frac{2k}{h} +1\right) - \ln \frac{\eta}{9}}{k}}
	\end{equation}
	with probability $1-\eta$. 
\end{theorem}  
It thus suffices to increase the number of data sets slightly faster than the VC dimension. 

To illustrate how to apply Theorem~\ref{thm:VCbound} we recall the class of
polytrees in Lemma~\ref{lem:polytrees}.
An interesting property of polytrees
is that every pair of non-adjacent nodes can already be rendered conditional independent by one appropriate intermediate node. This is because there is always at most one (undirected) path
connecting them. Moreover, 
for any two nodes $X,Y$ that are not too close together
in the DAG, there 
is a realistic chance
that some randomly chosen $Z$ satisfies
$X\independent Y\,|Z$. 
Therefore, we consider the following scenario:

\begin{enumerate}
	\item
	Draw $k$ triples
	$(Y_1,Y_2,Y_3)$ uniformly at random and check whether
	$
	Y_1 \independent Y_2\,| Y_3.
	$
	\item Search for a polytree $G$
	that is consistent with the $k$ observed \\
	(in)dependences.
	\item Predict conditional independences
	for unobserved triples via $G$
\end{enumerate}


Since the number of points in the training set should increase slightly faster than
the VC dimension (which is $O(n\log n)$, see Lemma~\ref{lem:polytrees}),
we know that a small fraction of
the possible independence tests (which grows with third power) is already sufficient to predict further conditional  independences.

The red curve in Figure~\ref{fig:numberOfTests} provides a rough estimate of how $k$ needs to grow if we want to ensure that the term $\sqrt{.}$ in
\eqref{eq:VCbinarygeneral} is below $0.1$ for $\eta =0.1$. 
The blue curve shows how the number of possible tests grows, which significantly exceeds the required ones after $n=40$.
\begin{figure}[thp]
	\centerline{
		\includegraphics[width=0.8\textwidth]{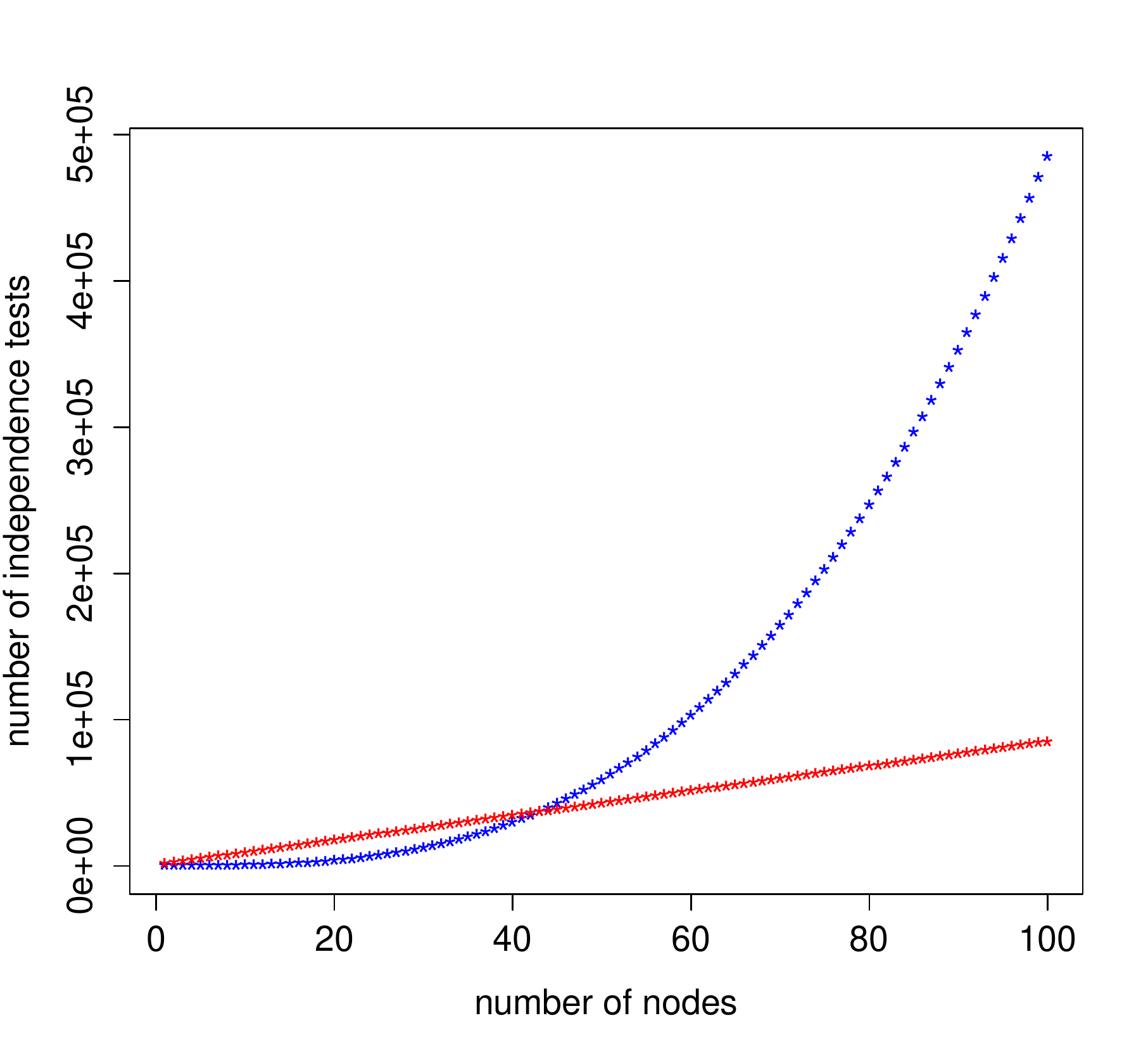}
	}
	\caption{\label{fig:numberOfTests}  The red curve shows how the number of tests required by the VC bound grows with the number of variables, while the blue one shows how the number of possible tests grows. 
	Thus, the bounds are getting useful at about $50$ nodes, where the curves cross.}

\end{figure}
For more than $100$ variables, only a fraction of about $1/4$ of the possible tests is needed to predict that also the remaining ones will hold with high probability. In \cref{subsec:predicting_inds} we will also look at a more practical example of how to apply \cref{thm:VCbound} to conditional independence tests.

While conditional independences have been used for causal inference already since decades, more recently it became popular to use other properties of distributions to infer causal DAGs.
In particular, several methods have been proposed that distinguish between cause and effect from bivariate distributions, e.g.,
\cite{Kano2003,Hoyer,Zhang_UAI,deterministic,discreteAN,Lopez2015,Mooij2016}. 
It is tempting to do multivariate causal inference by 
finding DAGs that are consistent with the bivariate causal direction test.
This motivates the following example.
\begin{lemma}[bivariate directionality test on DAGs]
	\label{lem:bivariate_direction}
	Let $\cG$ be the class of DAGs on $n$ nodes.
	Define a model-induced property $Q_G$ by
	\[
	Q_G(X_i,X_j):=
	\begin{cases}
		1 \text{ if there is a directed path } X_i\to X_j \text{ in } G\\
		0  \text{ else}
	\end{cases} 
	\]
	The VC dimension of $(Q_G)_{G\in \cG}$ is 
	at most $n-1$. 
\end{lemma}
\begin{proof}
	The VC dimension is the maximal number $h$ of 
	pairs of variables for which the causal directions can be oriented in all $2^h$ possible ways. 
	If we take $n$  or more pairs, the undirected graph 
	defined by connecting each pair contains a cycle 
	\[
	(X_1,X_2),(X_2,X_3),\dots,
	(X_{l-1},X_l),(X_l,X_1),
	\]
	with $l\leq n$.
	Then, however, not all $2^l$ causal directions are possible because 
	\[
	X_1 \to X_2 \to \cdots X_l \to X_1
	\]
	would be a directed cycle. Thus the VC dimension is smaller than $n$.
\end{proof}
This result can be used to infer causal directions for pairs that have not been observed together:
\begin{enumerate}
	
	\item Apply the bivariate causality test $Q_T$ to $k$ randomly chosen ordered pairs, where $k$ needs to grow slightly faster than $n$.
	
	\item Search for a DAG $G\in \cG$ that is consistent with many of the outcomes.
	
	\item Infer the outcome of further bivariate causality tests from $G$. 
	
\end{enumerate}

It is remarkable that the generalization bound holds regardless of how bivariate causality is tested and whether one understands which statistical features are used to infer the causal direction.
Solely the fact that a causal hypothesis from a class of low VC dimension matches the majority of the bivariate tests ensures that it generalizes well to future tests.

\subsection{Real-valued properties}

The VC bounds in Subsection~\ref{subsec:VCboundsbinary} referred to binary statistical properties. To consider
also real-valuedd properties note
that 
the VC dimension of a class of real-valued functions $(f_\lambda)_{\lambda\in \Lambda}$ with $f:\cX \rightarrow \R$
is defined as the VC dimension of
the set of binary functions, see Section~3.6 \cite{Vapnik1995}:
\[
\big(f_\lambda^{-1}\left((-\infty, r]\right)\big)_{\lambda \in \Lambda,r \in \R}.
\]
By combining (3.15) with (3.14) and (3.23) in \cite{Vapnik1995} we obtain:
\begin{theorem}[VC bound for real-valued statistical properties] Let $(Q_M)_{M\in \cM}$
	be a class of $[A,B]$-valued model-induced 
	properties with VC dimension $h$. 
	Given $k$ data sets $D_1,\dots,D_k$ sampled from some distribution $P_D$. Then
	\begin{align*}
	\Exp[|Q_T(D)-Q_M(D)|] &\leq    \frac{1}{k} \sum_{j=1}^k |Q_T(D_j)-Q_M(S_{D_j}) \\
	&+ (B-A) \sqrt{\frac{h\left(\ln \frac{k}{h}+1\right) -\ln \frac{\eta}{4}}{k}} \\
	\end{align*}
	with probability at least $1-\eta$.
\end{theorem}
This bound can easily be applied to the prediction of correlations via collider-free paths: Due to Lemma~\ref{Lem:VCcorrPath}, we then have $h\in O(n)$. Since correlations are in $[-1,1]$, we can set $2$  for $B-A$.

\subsection{Remarks on exchangeability required for learning theory}
\label{subsec:iid}
In practical applications, the scenario is
usually somehow different because
one does not choose `observed' and `unobserved' subsets randomly in a way that justifies 
exchangeability of data sets. Instead, the observed sets are defined by the available data sets. One may object that
the above considerations are therefore inapplicable. There is no formal argument against this objection. However, there may be reasons to believe that the observed variable sets at hand are not substantially different from the unobserved ones whose properties are supposed to be predicted,
apart from the fact that they are observed. Based on this belief, one may still use the above generalization bounds 
as guidance on the richness of the class of causal hypotheses that is allowed to obtain good generalization properties.

\section{Experiments}
\label{sec:experiments}
In this section we 
use simulated toy scenarios that illustrate how statistical properties of subsets of variables can be predicted via the detour of
inferring a causal model from a small class.

In the first scenario, we want to interpret a DAG as a model for conditional independences as in \cref{ex:dagsci} and use the classical PC algorithm \citep{Spirtes1993} to estimate a DAG from data.
In this setting \cref{thm:VCbound} provides guarantees for the accuracy of our model on conditional independence tests that have not been used for the construction of the DAG.
In the second scenario, we interpret polytrees as models for the admissibility of additive noise models similar to \cref{ex:daglingam}.
Again, we could interpret this scenario such that a causal discovery algorithm has \emph{in principle} access to all pairs of variables but does not use all of them (e.g. due to computational constraints).
Alternatively, we can interpret this scenario as the problem of merging marginal distributions in the sense of `integrative causal inference' \citep{Tsamardinos}, also when marginal distributions of some subsets are unavailable due to missing data.  

\subsection{Predicting independences}
\label{subsec:predicting_inds}
In \cref{ex:dagsci} we interpreted a DAG as a model of conditional independences, i.e. we defined $Q_G[(Y_1, \dots, Y_k)] = 1$ if the Markov condition implies $Y_1 \ind Y_2 \mid Y_3, \dots, Y_k$ and else $0$.
Given that there is a graph $G^*$ such that the joint distribution of the data $P_\bX$ is Markovian to $G^*$ and the tests $Q_T$ correctly output the independences, the empirical risk $\mathbb{E}[|Q_T(D) - Q_G(D)|]$ becomes zero for any $G$ in the Markov equivalence class of $G^*$.
In order to find this equivalence class, one could conduct all possible independence tests and construct the equivalence class from them.
The PC algorithm is more efficient and can recover the underlying equivalence class with a polynomial number of conditional independence tests for sparse graphs \citep{kalisch2007estimating}.
Further, in the limit of infinite data the result of the PC algorithm will perfectly represent the conditional independences used during the algorithm.
In this sense we want to interpret the PC algorithm as an ERM algorithm, that aims to minimize the empirical risk 
\begin{equation}
\label{eq:empirical_loss}
	\frac{1}{k} \sum_{i=1}^k|Q_T(D_i) - Q_G(D_i)|,
\end{equation} 
where $k$ is bounded by a polynomial in $n$.
It is important to note, that technically in this scenario \cref{thm:VCbound} does not hold, as the samples $D_1, \dots, D_k$ are not chosen independently.
We see this as an opportunity to test the conjecture that in this case the available variable sets do not substantially differ from the unseen ones, similar to what we have described in  \cref{subsec:iid}. 

\paragraph{Data generation}
In our experiments we synthetically generated linear structural models as ground truth.
First, we uniformly chose an order $\pi:\{1, \dots, n\}\to \{1, \dots, n\}$ of the variables and for each pair of nodes $X_i, X_j$ we add an edge $X_i \to X_j$ if $\pi(i) < \pi(j)$ and with probability $p$.
For each edge $X_i\to X_j$ we draw a structural coefficient $a_{i, j}$ uniformly from $[0.1, 1)$ and set all other $a_{i', j'} = 0$.
Then every value  $x_j$ of a variable $X_j$ is a linear combination of the values of previous variables and some noise
\begin{displaymath}
	x_j = n_j + \sum_{\mathclap{k(i) < k(j)}} a_{i, j}\cdot x_i,
\end{displaymath}
where the noise terms $n_j$ are all drawn independently from a standard normal distribution.
In all experiments we choose $p$ such that the expected degree of a node $\frac{2}{n}\frac{p(n^2-n)}{2} = 1.5$.

The PC algorithm has one hyperparameter, namely the confidence level of the conditional independence tests.
We randomly generate 10 datasets with 10, 20 and 40 nodes respectively as described above and each one with 30.000 samples.\footnote{Note, that even for this large sample size we still cannot hope to always decide correctly, as there is a non-negligible chance to have \enquote{almost} non-faithful distributions \citep{uhler2013}}
We conducted the Fisher-$Z$ test for partial correlations for all triplets $(X_i, X_j, X_k)$ and $(X_i, X_k, \emptyset)$ and compared the result with the graphical ground truth.
For each dataset we calculated the $F_1$-score and chose the confidence level with the maximal average score, which was $0.001$ in this case.

\paragraph{Experimental setup}
In this experiment, we want to see how close the empirical performance (the performance on the tuples $D_i = (Y_1, \dots, Y_k)$ used to construct the DAG) is to the expected performance (the performance on all possible tuples).
Due to computational constraints, we restrict ourselves to the case where $k \le 3$, i.e. we condition on at most one variable.
We calculate the empirical loss as in \cref{eq:empirical_loss}, where $Q_T$ denotes the statistical test results and $Q_G$ $d$-separation in $G$.
As the PC algorithm only outputs a partially directed graph, we randomly draw a DAG from the corresponding Markov-equivalence class to get a model $G$.
It can happen though, that the output of the PC algorithm is a PDAG that does not describe an equivalence class of DAGs.
In this case we randomly orient conflicting edges.
For the expected error we conduct the conditional independence test on all possible triples and tuples (i.e. $k\in \{2, 3\}$).

The differences between empirical risk and expected risk for different datasets are plotted in \cref{fig:exp_inds}.
For $n=20, 40, 100$ there are 20 datasets respectively.
We also plotted the theoretical error bound from \eqref{eq:VCbinarygeneral}.
Note that we rescaled the bounds by the factor $0.6$ for visualization purposes.
We can see that the empirical risk is closer to the expected risk for instances for larger numbers of nodes when the PC algorithm uses more conditional independence tests to construct the graph.

\begin{figure}[htp]
	\centering
	\includegraphics[width=.6\textwidth]{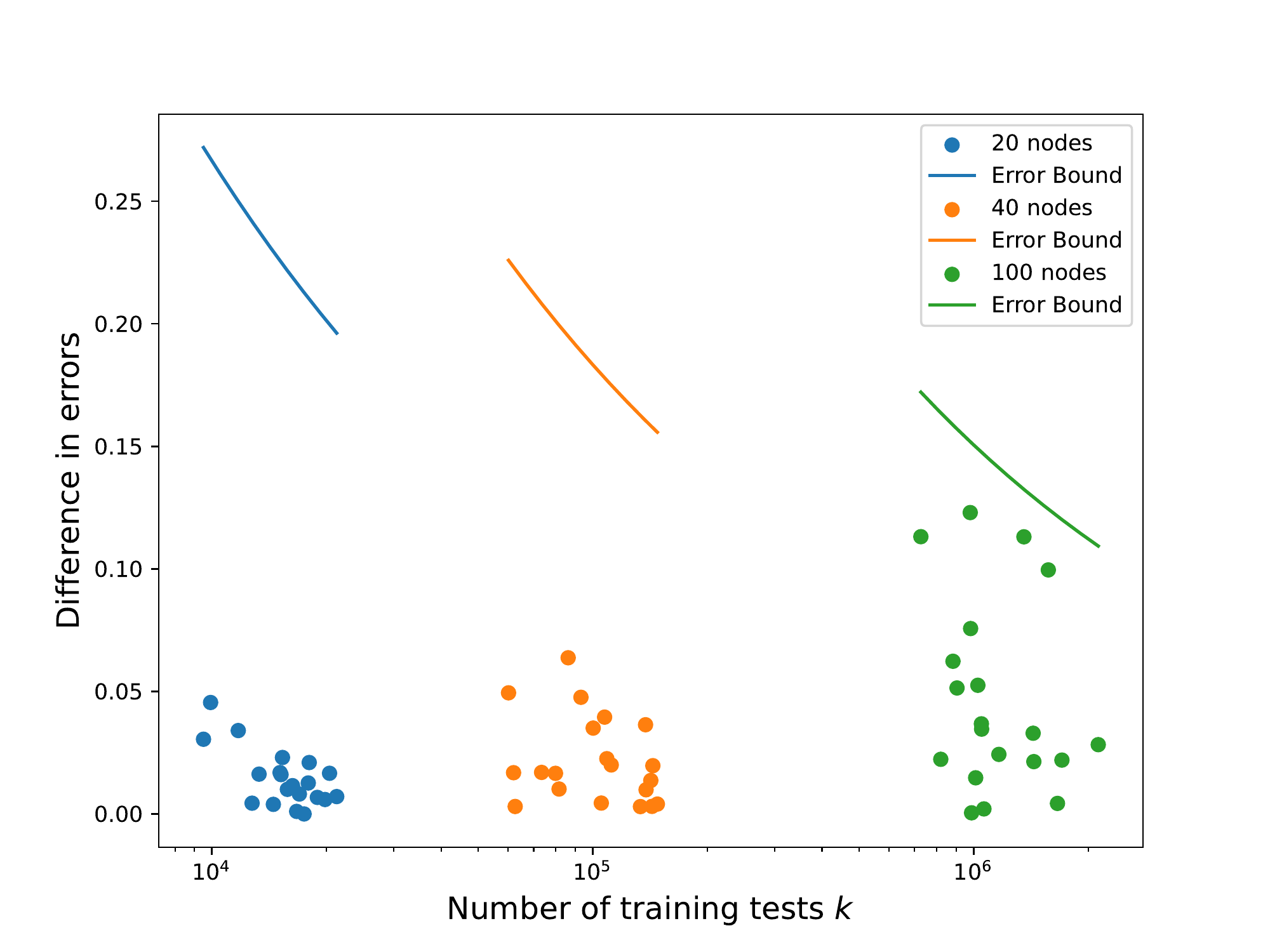}

	\caption{The difference between empirical error and expected error ($|\frac{1}{k} \sum_{i=1}^k|Q_T(D_i) - Q_G(D_i)| - 
	\mathbb{E}|Q_T(D) - Q_G(D)||$) versus the VC error bound when graphs with $n=20, 40, 100$ nodes are used as predictors for conditional independences. 
	}

	\label{fig:exp_inds}
\end{figure}

\subsection{Predicting existence of additive noise models (ANMs)}
In the next experiment we want to present a concrete example 
that constructs a simple DAG, namely a polytree, based on bivariate information (motivated by Lemma \ref{lem:bivariate_direction}) which is then used to infer bivariate statistical properties. 
We will use non-linear additive noise models  \citep{Hoyer} since they have achieved reasonable results 
in bivariate causal discovery \citep{Mooij2016}. We recall that $P_{X,Y}$ is said to admit an additive noise model from $X$ to $Y$ if there exists a
function $f$ such that the residual $Y- f(X)$ is independent of $X$. 


\paragraph{Generative model}
In this experiment we will generate the joint distribution via generalized additive models, i.e. we assume that for each node $X_i$ there are non-linear functions $f_{i, j}$ such that the values $x_i$ of $X_i$ are given by
	$$x_i = \sum_{\mathclap{x_j\in pa(X_i)}}f_{i, j}(x_j) + n_i,$$
where $n_i$ is a value of the noise term $N_i $ (which is independent from the parent nodes $X_j\in PA(X_i)$). 
This ensures that even for nodes with multiple parents there is a bivariate ANM from each of its parents. 

\paragraph{Testing for a bivariate ANM}
The statistical property of interest will be, whether $X_i$ and $X_j$  admit an additive noise model. In other words, we want to predict, whether we can construct a noise node $\hat N_i$ such that $\hat N_i \ind X_j$ by regressing $X_i$ on $X_j$ and then calculating $\hat{N}_i = X_i - \hat{f}_i(X_j)$ , where $\hat f_i$ denotes the regression function.
We then define our statistical property as
\begin{displaymath}
	Q(X_i, X_j) = \begin{cases}
	1  \quad \text{ if }   X_i\not\ind X_j    \text{ and } \exists \hat f_i: X_j\ind (X_i - \hat{f}_i(X_j)) \\
	0  \quad \text{ else}
	\end{cases}
\end{displaymath}
where the condition $X_i\not\ind X_j$ rules out the trivial case, where the nodes are already independent without subtraction of the regression result. 
Accordingly, for our statistical test $Q_T$ we replace the existence of $\hat f_i$ with an estimated regression function and the independences with statistical independence tests.

\paragraph{Representing $Q$ via a polytree}
Now we want to argue, that this statistical property can be represented by a polytree.
First, we notice that under some conditions, the existence of an ANM is not transitive.
\begin{lemma}[non-transitivity of additive noise] 
	\label{lem:non_linear_anm}
	Let $X, Y, Z$ be random variables, with the following structural causal model 
	\begin{displaymath}
		Z = f_Z(Y) + N_Z, \quad Y = f_Y(X) + N_Y, \quad X = N_X,
	\end{displaymath}
	where the noise terms $N_i$ are jointly independent, the cumulative distribution function of $N_Y$ is strictly increasing and the characteristic function of $N_Z$ is non-zero almost everywhere (with respect to the Lebesgue measure).
	Let further $f_Z$ be invertible and monotonously increasing in $n_Y$. If $f_Y$ is not constant and $f_Z$ is non-linear such that the map 
	$x\mapsto f_Z(f_Y(x) +  n_Y)$ is differentiable and fulfils
	\begin{equation}
	\label{eq:condition_non_linear_scm}
		\exists x\in\cX, n_Y, n_Y' \in \mathcal{N}_Y: \frac{\partial}{\partial x}f_Z(f_Y(x) +  n_Y) \neq \frac{\partial}{\partial x} f_Z(f_Y(x) + n'_Y),
	\end{equation}
	where $\cX$ and $\mathcal{N}_Y$ denote the support of $X$ and $N_Y$ respectively, 
	then there is no additive noise model from $X$ to $Z$.
\end{lemma}
The proof can be found in Appendix~A. 
Intuitively, Eq.~\ref{eq:condition_non_linear_scm} states that $f_Z$ is non-linear.
It also encodes the subtlety that this non-linearity must occur at a point, where additionally $f_Y$ is not constant and that is in the support of $Y$. 

In the following, we will assume without proof that concatenating more than two ANMs generically also does not result 
in an ANM (generalizing Lemma \ref{lem:non_linear_anm}). 
We further assume that variables $X_i$ and $X_j$ connected by a common cause do not admit an ANM, following known identifiability results
for multivariate ANMs in \cite{UAI_identifiability} with appropriate genericity conditions.   
Under this premise, a polytree contains an edge $X_i\to X_j$ if and only if  $Q_T(X_i, X_j) = 1$.
Motivated by this connection, we define 
	\[
Q_G(X_i,X_j):=
\begin{cases}
	1 \text{ if there is an edge } X_i\to X_j \text{ in } G\\
	0  \text{ else}
\end{cases} 
\]

The reader might wonder why we state assumptions about the generative process of the data in the section above, even though we have repeatedly emphasized that our theory does not need to reference a `true' causal model.
Note, that we primarily used Lemma~\ref{lem:non_linear_anm} to render polytrees a well-defined model for $Q_T$ in the sense of Definition~\ref{def:model}.
For the validity of \cref{thm:VCbound} it does not matter whether the data has actually been generated by an additive noise model (although we might be able to achieve a lower empirical risk in that case). 

\paragraph{Causal discovery algorithm}
We then estimate a graph $G$ with the following procedure.
\begin{enumerate}
	
	\item Apply the bivariate causality test $Q_T$ to $k$ randomly chosen ordered pairs. To estimate the functions $f_i$, we use the Gaussian Process implementation from sklearn \citep{scikit-learn} and to test independences we use the kernel independence test \citep{zhang2011kernel} as implemented by \citet{blobaum2022dowhy}.
	
	\item We add an edge if $Q_T(X_i, X_j) = 1$. 
	
	\item If the resulting graph is not a tree, in each undirected cycle we remove the edge, where $X_i\ind \hat N_i$ has the lowest $p$ value.
	
\end{enumerate}

\paragraph{Experimental setup}
We generate five causal graphs analogously to \cref{subsec:predicting_inds}, with the additional constraint that we do not add edges, if they would close an undirected circle. 
To generate the generalised additive mechanisms $f_{i, j}$ we use neural networks with a single hidden layer with 20 nodes, $\tanh$ activation function and uniformly random weights from $[0.1, 1)$.
Moreover, we used uniform instead of Gaussian noise.
All causal graphs in the experiment contain 20 nodes.
We draw 600 samples and use 0.05 as confidence level.
For each dataset, we draw $k$ tuples of variables and use the algorithm described above to estimate a polytree.
The plot in Figure~\ref{fig:exp_anm} shows difference between empirical error and the expected error for increasing $k$ on the same dataset, as well as the theoretical generalization bound from Eq~\ref{eq:VCbinarygeneral}.
Note, that the expectation is to be understood w.r.t. to tuples of variables.
This means that the expected error is simply the prediction error evaluated on all possible tuples of variables.
Also note, that we rescaled the bound by the factor $0.2$ for visualization purposes.
We repeated the causal discovery 20 times for each joint dataset (but with different randomly drawn marginals) and plotted the mean and the 90\% empirical quantile.
The difference
turns out to be small when the number of training tuples is large, in agreement with statistical learning theory.  
In Appendix B we provide additional plots with other datasets drawn according to the above procedure, to demonstrate that the results in Figure~\ref{fig:exp_anm} are not due to a peculiar ground truth model.

\begin{figure}[htp]
	\centering
	\includegraphics[width=.6\textwidth]{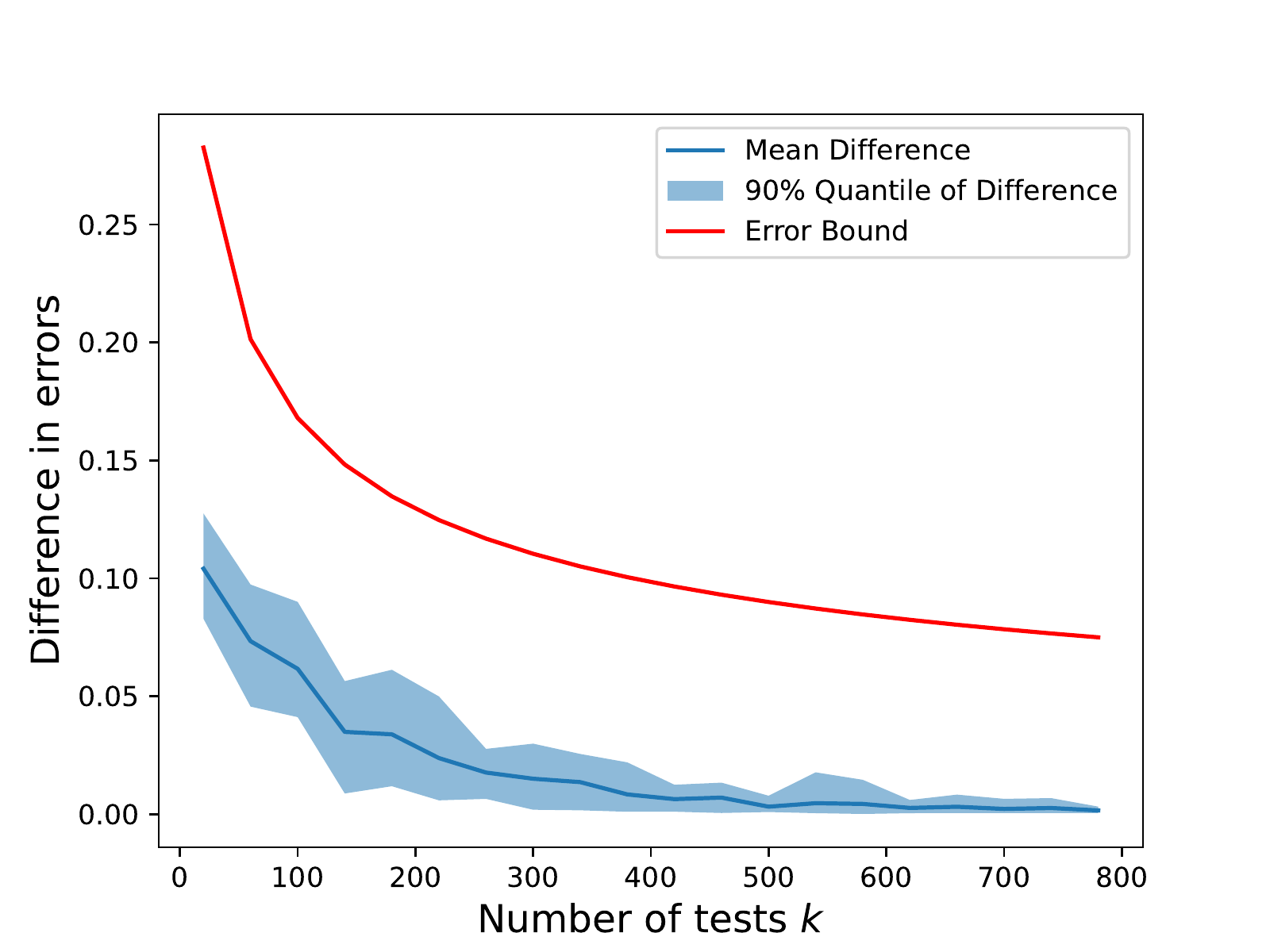}
	\caption{The absolute difference between empirical error $\frac{1}{k} \sum_{i=1}^k|Q_T(D_i) - Q_G(D_i)|$ and expected error $\mathbb{E}(|Q_T(D) - Q_G(D)|)$ for graphs with $n=20$ nodes when $k$ tuples of variables are used to fit a polytree.}
	\label{fig:exp_anm}
\end{figure}

\begin{remark}
	Unfortunately, inferring causal directions using the class of additive noise models raises the following dilemma:
	the class 
	is not 
	closed  under marginalization, e.g., if there is a (non-linear) ANM from $X$ to $Y$ and one from $Y$ to $Z$, there is, in the generic case, no 
	ANM from $X$ to $Z$, as we argued in Lemma~\ref{lem:non_linear_anm}. 
	For this reason, it would  not be mathematically consistent to use the ANM condition for inferring whether there is a directed path from a variable $X_i$ to $X_j$. Instead, checking ANM infers whether $X_i$ is a {\it direct} cause of $X_j$. The model class ANM thus suggests an {\it absolute} distinction 
	between `direct' and `indirect' influence, while in nature the distinction always refers to the set of observed variables (since we can always 
	zoom into the mechanism transmitting the information between the variables).  We will, however, accept this artificial distinction between 
	direct and indirect to get a mathematically consistent toy example. 
\end{remark}

\section{Revisiting Common Problems of Causal Discovery}
\label{sec:common_problems}
We now explain how our interpretation of causal models as predictors of statistical properties provides a slightly different perspective,
both on conceptual and practical questions of causal discovery. 
 
\paragraph{Predicting impact of interventions by merging distributions \label{sec:interventions}}

We have argued that causal hypotheses
provide strong guidance on how to merge probability distributions and thus
become empirically testable without resorting to interventions.
One may wonder whether this view on causality is completely disconnected to interventions.
Here we argue that it is not.
In some sense, estimating the impact of an intervention can also be phrased as the problem of inferring properties of unobserved joint distributions.

Assume we want to test whether the causal hypothesis $X\to Y$ is true.
We would then check how the distribution of $Y$ changes under randomized interventions on $X$.
Let us formally introduce a variable $F_X$ \citep{Pearl2000} that
can attain all possible values $x$ of $X$
(indicating to which value $x$ is set to)
or the value ${\tt idle}$ (if no intervention is made).
Whether $X$ influences $Y$ is then equivalent to
\begin{equation}\label{eq:notind}
F_X \not\independent Y.
\end{equation}
If we demand that this causal relation is unconfounded (as is usually intended by the notation $X\to Y$), we have to test the condition
\begin{equation}\label{eq:notconf}
P_{Y|F_X=x} = P_{Y|X=x}.
\end{equation}
Before the intervention is made,
both conditions \eqref{eq:notind} and \eqref{eq:notconf} refer to the unobserved distribution $P_{Y,F_X}$.
Inferring whether $X\to Y$ is true from $P_{X,Y}$ thus amounts to inferring
the unobserved distribution $P_{Y,F_X}$
from $P_{X,Y}$ plus the additional background knowledge regarding the statistical and causal relation between $F_X$ and $X$ (which is just based on the knowledge that the action we made has been in fact the desired intervention).
In applications it can be a non-trivial question why some action can be considered an intervention on a target variable at hand (for instance in complex gene-gene interactions). 
If one assumes that it is based on purely observational data (maybe earlier in the past),
we have reduced the problem of predicting the impact of interventions entirely to the problem of merging joint distributions.

\paragraph{Linear causal models for non-linear relations}
Our perspective justifies to
apply multivariate Gaussian causal models to data sets that are clearly non-Gaussian:
Assume a hypothetical causal graph is inferred from the conditional independence pattern obtained via {\it partial correlation tests} (which is correct only for multivariate Gaussians), as done by common causal inference software
\cite{TETRAD}. Even if one knows that
the graph only represents
partial correlations correctly, but not conditional independences, it
may predict well partial correlations
of unseen variable sets.
This way, the linear causal model can be helpful when the goal is only to predict linear statistics. This is good news
particularly because general conditional independence tests remain a difficult issue \citep{Shah2000}.

\paragraph{Tuning of confidence levels}
There is also another heuristic solution of a difficult question in causal inference that can be justified:
Inferring causal DAGs based on causal Markov condition and causal faithfulness \citep{Spirtes1993} relies on setting
the confidence levels for accepting conditional dependence. In practice, one will usually adjust the level such
that enough independences are accepted and enough are rejected for the sample size at hand.
Too few independences will results in a maximal DAG, too many in a graph with no edges.
Adjusting the confidence level is problematic, however, from the perspective of the common justification of causal faithfulness: if one rejects causal hypotheses with accidental conditional
independences because they occur `with measure zero' \citep{Meek1995},
it becomes questionable to set the confidence level high enough just because one wants ot get some independences accepted.\footnote{For a detailed discussion of how causal conclusions
of several causal inference algorithms
may repeatedly change after increasing the  sample size see \citep{Kelly2010}.}

Here we argue as follows instead:
Assume we are given any arbitrary confidence level as threshold for the conditional independence tests. Further assume we have  found
a DAG $G$ from a sufficiently small model class
that is consistent with all
the outcomes 'reject/accept' of
the conditional independence tests on a large number of subsets $S_1,\dots,S_k$. It is then justified to assume that $G$ will correctly predict the outcomes of this test
for unobserved variable sets $\tilde{S_1},\dots,\tilde{S}_l \subset S_1\cup \cdots \cup S_k$
for this particular confidence level.
This is because $G$ predicts the outcomes of the tests, not the properties themselves, just as in \cref{ex:gt_vs_pred}.

\paragraph{Methodological justification of causal faithfulness}
In our learning scenarios, DAGs are used to predict for some
choice of variables $X_{j_1},X_{j_2},\dots,X_{j_k}$
whether
\[
X_{j_1} \independent X_{j_2}\,|X_{j_3}, \dots,X_{j_k}.
\]
Without faithfulness, the DAG can only entail {\it in}dependence, but never entail dependence.
Rather than stating that `unfaithful distributions are unlikely' we need faithfulness
simply to obtain a definite prediction in the first place. 
This way, we avoid discussions about whether violations of faithfulness occur with probability zero 
(relying on assuming probability densities in parameter space \citep{Meek}, which has been criticized by \cite{LemeireJ2012}). 
After all, the argument  
is problematic  for finite data because distributions with weak dependences are not unlikely for DAGs with many nodes \citep{uhler2013}.   
Regardless of whether one believes that distributions in nature are faithful with respect to the `true DAG', any DAG that explains 
a sufficiently large set of dependences and independences is likely to also predict future (in)dependences.

\section{Conclusions}

We have described different scenarios where causal models can be used to infer statistical properties
of joint distributions of variables that have never been observed together. If 
the causal models are taken from a class of
sufficiently low VC dimension, this can be justified by generalization bounds from statistical learning theory. 

This opens a new pragmatic and context-dependent perspective on causality
where the essential empirical content of a causal model may consist in its prediction
regarding how to merge distributions from overlapping data sets. Such a pragmatic use of
causal concepts may be helpful for domains where 
the interventional definition of causality
raises difficult questions
(if one claims that the age of a person causally influences his/her income, as assumed in \citet{Mooij2016}, it is unclear  
what it means to intervene on the variable 'Age'). We have, moreover, argued that 
our pragmatic view of causal models is related to the usual concept of causality in terms of
interventions. 



\begin{thebibliography}{62}
\providecommand{\natexlab}[1]{#1}
\providecommand{\url}[1]{\texttt{#1}}
\expandafter\ifx\csname urlstyle\endcsname\relax
  \providecommand{\doi}[1]{doi: #1}\else
  \providecommand{\doi}{doi: \begingroup \urlstyle{rm}\Url}\fi

\bibitem[Pearl(2000)]{Pearl2000}
J.~Pearl.
\newblock \emph{Causality: Models, reasoning, and inference}.
\newblock Cambridge University Press, 2000.

\bibitem[Spirtes et~al.(1993)Spirtes, Glymour, and Scheines]{Spirtes1993}
P.~Spirtes, C.~Glymour, and R.~Scheines.
\newblock \emph{Causation, Prediction, and Search}.
\newblock Springer-Verlag, New York, NY, 1993.

\bibitem[Janzing and Mejia(2022)]{janzing2022phenomenological}
Dominik Janzing and Sergio Hernan~Garrido Mejia.
\newblock Phenomenological causality.
\newblock {\tt preprint arXiv:2211.09024}, 2022.

\bibitem[Sch\"olkopf et~al.(2012)Sch\"olkopf, Janzing, Peters, Sgouritsa,
  Zhang, and Mooij]{anticausal}
B.~Sch\"olkopf, D.~Janzing, J.~Peters, E.~Sgouritsa, K.~Zhang, and J.~Mooij.
\newblock On causal and anticausal learning.
\newblock In Langford J. and J.~Pineau, editors, \emph{Proceedings of the 29th
  International Conference on Machine Learning (ICML)}, pages 1255--1262. ACM,
  2012.

\bibitem[Danks(2005)]{danks2005scientific}
David Danks.
\newblock Scientific coherence and the fusion of experimental results.
\newblock \emph{The British Journal for the Philosophy of Science}, 2005.

\bibitem[Danks et~al.(2008)Danks, Glymour, and Tillman]{danks2008integrating}
David Danks, Clark Glymour, and Robert Tillman.
\newblock Integrating locally learned causal structures with overlapping
  variables.
\newblock \emph{Advances in Neural Information Processing Systems}, 21, 2008.

\bibitem[Claassen and Heskes(2010)]{claassen2010causal}
Tom Claassen and Tom Heskes.
\newblock Causal discovery in multiple models from different experiments.
\newblock \emph{Advances in Neural Information Processing Systems}, 23, 2010.

\bibitem[Tsamardinos et~al.(2012)Tsamardinos, Triantafillou, and
  Lagani]{Tsamardinos}
I.~Tsamardinos, S.~Triantafillou, and V.~Lagani.
\newblock Towards integrative causal analysis of heterogeneous data sets and
  studies.
\newblock \emph{J. Mach. Learn. Res.}, 13\penalty0 (1):\penalty0 1097--1157,
  2012.

\bibitem[Triantafillou and Tsamardinos(2015)]{triantafillou2015constraint}
Sofia Triantafillou and Ioannis Tsamardinos.
\newblock Constraint-based causal discovery from multiple interventions over
  overlapping variable sets.
\newblock \emph{The Journal of Machine Learning Research}, 16\penalty0
  (1):\penalty0 2147--2205, 2015.

\bibitem[Huang et~al.(2020)Huang, Zhang, Gong, and Glymour]{huang2020causal}
Biwei Huang, Kun Zhang, Mingming Gong, and Clark Glymour.
\newblock Causal discovery from multiple data sets with non-identical variable
  sets.
\newblock In \emph{Proceedings of the AAAI conference on artificial
  intelligence}, volume~34, pages 10153--10161, 2020.

\bibitem[Kano and Shimizu(2003)]{Kano2003}
Y.~Kano and S.~Shimizu.
\newblock Causal inference using nonnormality.
\newblock In \emph{Proceedings of the International Symposium on Science of
  Modeling, the 30th Anniversary of the Information Criterion}, pages 261--270,
  Tokyo, Japan, 2003.

\bibitem[Hoyer et~al.(2009)Hoyer, Janzing, Mooij, Peters, and
  Sch\"olkopf]{Hoyer}
P.~Hoyer, D.~Janzing, J.~Mooij, J.~Peters, and B~Sch\"olkopf.
\newblock Nonlinear causal discovery with additive noise models.
\newblock In D.~Koller, D.~Schuurmans, Y.~Bengio, and L.~Bottou, editors,
  \emph{Proceedings of the conference Neural Information Processing Systems
  (NIPS) 2008}, Vancouver, Canada, 2009. MIT Press.

\bibitem[Peters et~al.(2017)Peters, Janzing, and Sch\"olkopf]{causality_book}
J.~Peters, D.~Janzing, and B.~Sch\"olkopf.
\newblock \emph{Elements of Causal Inference -- Foundations and Learning
  Algorithms}.
\newblock MIT Press, 2017.

\bibitem[Vorob'ev(1962)]{Vorobev1962}
N.~Vorob'ev.
\newblock Consistent families of measures and their extensions.
\newblock \emph{Theory Probab. Appl}, 7\penalty0 (2):\penalty0 147--163, 1962.

\bibitem[Kellerer(1964)]{Kellerer1964}
H.~Kellerer.
\newblock {Ma{\ss}theoretische Marginalprobleme}.
\newblock \emph{Math. Ann.}, 153:\penalty0 168--198, 1964.
\newblock in German.

\bibitem[Janzing(2018)]{janzing2018merging}
Dominik Janzing.
\newblock Merging joint distributions via causal model classes with low vc
  dimension.
\newblock {\tt arXiv preprint arXiv:1804.03206}, 2018.

\bibitem[Janzing(2016)]{causalMarginalTalk}
D.~Janzing.
\newblock From the probabilistic marginal problem to the causal marginal
  problem.
\newblock talk in the open problem session of the workshop `Causation:
  Foundation to Application' of the Conference on Uncertainty in Artificial
  Intelligence (UAI), 2016.
\newblock \url{people.hss.caltech.edu/~fde/UAI2016WS/talks/Dominik.pdf}.

\bibitem[Gresele et~al.(2022)Gresele, K{\"u}gelgen, K{\"u}bler, Kirschbaum,
  Sch{\"o}lkopf, and Janzing]{Gresele2022}
Luigi Gresele, Julius~Von K{\"u}gelgen, Jonas K{\"u}bler, Elke Kirschbaum,
  Bernhard Sch{\"o}lkopf, and Dominik Janzing.
\newblock Causal inference through the structural causal marginal problem.
 \emph{Proceedings of the 39th
  International Conference on Machine Learning}, volume 162 of
  \emph{Proceedings of Machine Learning Research}, pages 7793--7824. PMLR,
2022.

\bibitem[Guo et~al.(2023)Guo, Wildberger, and Schölkopf]{Guo2023}
Siyuan Guo, Jonas Wildberger, and Bernhard Schölkopf.
\newblock Out-of-variable generalization.
\newblock \emph{preprint {\tt arXiv:2304.07896}}, 2023.

\bibitem[Haussler(1992)]{haussler1992decision}
David Haussler.
\newblock Decision theoretic generalizations of the pac model for neural net
  and other learning applications.
\newblock \emph{Information and computation}, 100\penalty0 (1):\penalty0
  78--150, 1992.

\bibitem[Bartlett and Long(2021)]{bartlett2021failures}
Peter~L Bartlett and Philip~M Long.
\newblock Failures of model-dependent generalization bounds for least-norm
  interpolation.
\newblock \emph{The Journal of Machine Learning Research}, 22\penalty0
  (1):\penalty0 9297--9311, 2021.

\bibitem[Belkin et~al.(2019{\natexlab{a}})Belkin, Hsu, Ma, and
  Mandal]{Bel:2019}
Mikhail Belkin, Daniel Hsu, Siyuan Ma, and Soumik Mandal.
\newblock Reconciling modern machine-learning practice and the classical
  bias{\textendash}variance trade-off.
\newblock \emph{Proceedings of the National Academy of Sciences},
  2019{\natexlab{a}}.

\bibitem[Belkin et~al.(2019{\natexlab{b}})Belkin, Rakhlin, and
  Tsybakov]{belkin2019does}
Mikhail Belkin, Alexander Rakhlin, and Alexandre~B Tsybakov.
\newblock Does data interpolation contradict statistical optimality?
\newblock In \emph{International Conference on Artificial Intelligence and
  Statistics (AISTATS)}, 2019{\natexlab{b}}.

\bibitem[Liang and Rakhlin(2020)]{liang2020just}
Tengyuan Liang and Alexander Rakhlin.
\newblock Just interpolate: Kernel “ridgeless” regression can generalize.
\newblock \emph{The Annals of Statistics}, 2020.

\bibitem[Tsigler and Bartlett(2020)]{tsigler2020benign}
Alexander Tsigler and Peter~L Bartlett.
\newblock Benign overfitting in ridge regression.
\newblock {\tt preprint arXiv:2009.14286}, 2020.

\bibitem[Bartlett et~al.(2020)Bartlett, Long, Lugosi, and
  Tsigler]{bartlett2020benign}
Peter~L Bartlett, Philip~M Long, G{\'a}bor Lugosi, and Alexander Tsigler.
\newblock Benign overfitting in linear regression.
\newblock \emph{Proceedings of the National Academy of Sciences}, vol 117, 30063--30070, 2020.

\bibitem[Muthukumar et~al.(2020)Muthukumar, Vodrahalli, Subramanian, and
  Sahai]{muthukumar2020harmless}
Vidya Muthukumar, Kailas Vodrahalli, Vignesh Subramanian, and Anant Sahai.
\newblock Harmless interpolation of noisy data in regression.
\newblock \emph{IEEE Journal on Selected Areas in Information Theory}, 2020.

\bibitem[Sun et~al.(2006)Sun, Janzing, and Sch\"{o}lkopf]{SunLauderdale}
X.~Sun, D.~Janzing, and B.~Sch\"{o}lkopf.
\newblock {Causal inference by choosing graphs with most plausible Markov
  kernels}.
\newblock In \emph{Proceedings of the 9th International Symposium on Artificial
  Intelligence and Mathematics}, pages 1--11, Fort Lauderdale, FL, 2006.

\bibitem[Zhang and Hyv\"arinen(2009)]{Zhang_UAI}
K.~Zhang and A.~Hyv\"arinen.
\newblock On the identifiability of the post-nonlinear causal model.
\newblock In \emph{Proceedings of the 25th Conference on Uncertainty in
  Artificial Intelligence}, Montreal, Canada, 2009.

\bibitem[Daniusis et~al.(2010)Daniusis, Janzing, Mooij, Zscheischler, Steudel,
  Zhang, and Sch{\"o}lkopf]{deterministic}
P.~Daniusis, D.~Janzing, J.~M. Mooij, J.~Zscheischler, B.~Steudel, K.~Zhang,
  and B.~Sch{\"o}lkopf.
\newblock Inferring deterministic causal relations.
\newblock In \emph{Proceedings of the 26th Annual Conference on {U}ncertainty
  in {A}rtificial {I}ntelligence ({UAI})}, pages 143--150. AUAI Press, 2010.

\bibitem[Janzing et~al.(2009)Janzing, Sun, and Sch\"olkopf]{SecondOrder}
D.~Janzing, X.~Sun, and B.~Sch\"olkopf.
\newblock Distinguishing cause and effect via second order exponential models.
\newblock \emph{\em{\url{http://arxiv.org/abs/0910.5561}}}, 2009.

\bibitem[Stegle et~al.(2010)Stegle, Janzing, Zhang, Mooij, and
  Sch{\"o}lkopf]{stegle2010probabilistic}
Oliver Stegle, Dominik Janzing, Kun Zhang, Joris~M Mooij, and Bernhard
  Sch{\"o}lkopf.
\newblock Probabilistic latent variable models for distinguishing between cause
  and effect.
\newblock \emph{Advances in neural information processing systems}, 23, 2010.

\bibitem[Peters et~al.(2010)Peters, Janzing, and
  Sch{\"o}lkopf]{peters2010identifying}
Jonas Peters, Dominik Janzing, and Bernhard Sch{\"o}lkopf.
\newblock Identifying cause and effect on discrete data using additive noise
  models.
\newblock In \emph{Proceedings of the thirteenth international conference on
  artificial intelligence and statistics}, pages 597--604. JMLR Workshop and
  Conference Proceedings, 2010.

\bibitem[Mooij et~al.(2016)Mooij, Peters, Janzing, Zscheischler, and
  Sch\"olkopf]{Mooij2016}
J.~Mooij, J.~Peters, D.~Janzing, J.~Zscheischler, and B.~Sch\"olkopf.
\newblock Distinguishing cause from effect using observational data: methods
  and benchmarks.
\newblock \emph{Journal of Machine Learning Research}, 17\penalty0
  (32):\penalty0 1--102, 2016.

\bibitem[Marx and Vreeken(2017)]{Marx2017}
A.~Marx and J.~Vreeken.
\newblock Telling cause from effect using mdl-based local and global
  regression.
\newblock In \emph{2017 {IEEE} International Conference on Data Mining, {ICDM}
  2017, New Orleans, LA, USA, November 18-21, 2017}, pages 307--316, 2017.

\bibitem[Hoyer et~al.(2008)Hoyer, Shimizu, Kerminen, and
  Palviainen]{HoyerLatent08}
P.~Hoyer, S.~Shimizu, A.~Kerminen, and M.~Palviainen.
\newblock Estimation of causal effects using linear non-gaussian causal models
  with hidden variables.
\newblock \emph{International Journal of Approximate Reasoning}, 49\penalty0
  (2):\penalty0 362 -- 378, 2008.

\bibitem[Lopez-Paz et~al.(2015)Lopez-Paz, Muandet, Sch{\"o}lkopf, and
  Tolstikhin]{Lopez2015}
D.~Lopez-Paz, K.~Muandet, B.~Sch{\"o}lkopf, and I.~Tolstikhin.
\newblock Towards a learning theory of cause-effect inference.
\newblock In \emph{Proceedings of the 32nd International Conference on Machine
  Learning}, volume~37 of \emph{JMLR Workshop and Conference Proceedings}, page
  1452–1461. JMLR, 2015.

\bibitem[Tsamardinos et~al.(2006)Tsamardinos, Brown, and
  Aliferis]{tsamardinos2006max}
Ioannis Tsamardinos, Laura~E Brown, and Constantin~F Aliferis.
\newblock The max-min hill-climbing bayesian network structure learning
  algorithm.
\newblock \emph{Machine learning}, 65\penalty0 (1):\penalty0 31--78, 2006.

\bibitem[Gentzel et~al.(2019)Gentzel, Garant, and Jensen]{gentzel2019case}
Amanda Gentzel, Dan Garant, and David Jensen.
\newblock The case for evaluating causal models using interventional measures
  and empirical data.
\newblock \emph{Advances in Neural Information Processing Systems}, 32, 2019.

\bibitem[Vapnik(1998)]{Vapnik}
V.~Vapnik.
\newblock \emph{Statistical learning theory}.
\newblock John Wileys \& Sons, New York, 1998.

\bibitem[Lauritzen(1996)]{Lauritzen}
S.~Lauritzen.
\newblock \emph{Graphical Models}.
\newblock Clarendon Press, Oxford, New York, {Oxford Statistical Science
  Series} edition, 1996.

\bibitem[Aigner and Ziegler(1998)]{Aigner1998}
M.~Aigner and G.~Ziegler.
\newblock \emph{{Proofs from THE BOOK}}.
\newblock Springer, Berlin, 1998.

\bibitem[Radhakrishnan et~al.(2017)Radhakrishnan, Solus, and
  Uhler]{Radhakrishnan2017}
A.~Radhakrishnan, L.~Solus, and C.~Uhler.
\newblock {Counting Markov equivalence classes for DAG models on trees}.
\newblock \emph{Discrete Applied Mathematics}, vol 244, 170--185, 2018.

\bibitem[Vapnik(1995)]{Vapnik1995}
V.~Vapnik.
\newblock \emph{The nature of statistical learning theory}.
\newblock Springer, New York, 1995.

\bibitem[van~der Wart and Wellner(2009)]{vanDerVaart2009}
A.~van~der Wart and J.~Wellner.
\newblock {A note on bounds for VC dimensions}.
\newblock \emph{Inst Math Stat Collect}, 5:\penalty0 103--107, 2009.

\bibitem[Vapnik(2006)]{Vapnik06}
V.~Vapnik.
\newblock \emph{Estimation of Dependences Based on Empirical Data}.
\newblock Statistics for Engineering and Information Science. Springer Verlag,
  New York, 2nd edition, 2006.

\bibitem[Peters et~al.(2011)Peters, Janzing, and Sch{\"o}lkopf]{discreteAN}
J.~Peters, D.~Janzing, and B.~Sch{\"o}lkopf.
\newblock Causal inference on discrete data using additive noise models.
\newblock \emph{IEEE Transactions on Pattern Analysis and Machine
  Intelligence}, 33\penalty0 (12):\penalty0 2436--2450, 2011.

\bibitem[Kalisch and B{\"u}hlman(2007)]{kalisch2007estimating}
Markus Kalisch and Peter B{\"u}hlman.
\newblock Estimating high-dimensional directed acyclic graphs with the
  pc-algorithm.
\newblock \emph{Journal of Machine Learning Research}, 8\penalty0 (3), 2007.

\bibitem[Uhler et~al.(2013)Uhler, Raskutti, B\"uhlmann, and Yu]{uhler2013}
C.~Uhler, G.~Raskutti, P.~B\"uhlmann, and B.~Yu.
\newblock Geometry of the faithfulness assumption in causal inference.
\newblock \emph{The Annals of Statistics}, 41\penalty0 (2):\penalty0 436--463,
  04 2013.

\bibitem[Peters et~al.(2011)Peters, Mooij, Janzing, and
  Sch\"olkopf]{UAI_identifiability}
J.~Peters, J.~Mooij, D.~Janzing, and B.~Sch\"olkopf.
\newblock Identifiability of causal graphs using functional models.
\newblock In \emph{Proceedings of the 27th Conference on Uncertainty in
  Artificial Intelligence (UAI 2011)}.

\bibitem[Pedregosa et~al.(2011)Pedregosa, Varoquaux, Gramfort, Michel, Thirion,
  Grisel, Blondel, Prettenhofer, Weiss, Dubourg, Vanderplas, Passos,
  Cournapeau, Brucher, Perrot, and Duchesnay]{scikit-learn}
F.~Pedregosa, G.~Varoquaux, A.~Gramfort, V.~Michel, B.~Thirion, O.~Grisel,
  M.~Blondel, P.~Prettenhofer, R.~Weiss, V.~Dubourg, J.~Vanderplas, A.~Passos,
  D.~Cournapeau, M.~Brucher, M.~Perrot, and E.~Duchesnay.
\newblock Scikit-learn: Machine learning in {P}ython.
\newblock \emph{Journal of Machine Learning Research}, 12:\penalty0 2825--2830,
  2011.

\bibitem[Zhang et~al.(2011)Zhang, Peters, Janzing, and
  Sch{\"o}lkopf]{zhang2011kernel}
Kun Zhang, Jonas Peters, Dominik Janzing, and Bernhard Sch{\"o}lkopf.
\newblock Kernel-based conditional independence test and application in causal
  discovery.
\newblock In {\it Proceedings of the Twenty-Seventh Conference on Uncertainty
  in Artificial Intelligence}, pages 804--813, 2011.

\bibitem[Bl{\"o}baum et~al.(2022)Bl{\"o}baum, G{\"o}tz, Budhathoki, Mastakouri,
  and Janzing]{blobaum2022dowhy}
Patrick Bl{\"o}baum, Peter G{\"o}tz, Kailash Budhathoki, Atalanti~A Mastakouri,
  and Dominik Janzing.
\newblock Dowhy-gcm: An extension of dowhy for causal inference in graphical
  causal models.
\newblock {\tt preprint arXiv:2206.06821}, 2022.

\bibitem[TETRAD()]{TETRAD}
TETRAD.
\newblock The tetrad homepage.\\
\newblock \url{http://www.phil.cmu.edu/projects/tetrad/}.

\bibitem[Shah and Peters(2020)]{Shah2000}
Rajen~D Shah and Jonas Peters.
\newblock The hardness of conditional independence testing and the generalised
  covariance measure.
\newblock \emph{The Annals of Statistics}, 48\penalty0 (3):\penalty0
  1514--1538, 2020.

\bibitem[Meek(1995{\natexlab{a}})]{Meek1995}
C.~Meek.
\newblock Causal inference and causal explanation with background knowledge.
\newblock In \emph{Proceedings of the 11th Conference on Uncertainty in
  Artificial Intelligence}, pages 403--441, San Francisco, CA,
  1995{\natexlab{a}}. Morgan Kaufmann.

\bibitem[Kelly and Mayo-Wilson(2010)]{Kelly2010}
K.~Kelly and C.~Mayo-Wilson.
\newblock Causal conclusions that flip repeatedly.
\newblock In P.~Gr\"unwald and P.~Spirtes, editors, \emph{Proceedings of the
  Conference on Uncertainty in Artificial Intelligence (UAI 2010)}. AUAI Press,
  2010.

\bibitem[Meek(1995{\natexlab{b}})]{Meek}
C.~Meek.
\newblock {Strong completeness and faithfulness in Bayesian networks}.
\newblock \emph{{{\em Proceedings of} 11th Uncertainty in Artificial
  Intelligence (UAI), Montreal, Canada, {\em Morgan Kaufmann}}}, pages
  411--418, 1995{\natexlab{b}}.

\bibitem[Lemeire and Janzing(2012)]{LemeireJ2012}
J.~Lemeire and D.~Janzing.
\newblock Replacing causal faithfulness with algorithmic independence of
  conditionals.
\newblock \emph{Minds and Machines}, 23\penalty0 (2):\penalty0 227--249, 7
  2012.

\bibitem[Chao et~al.(2023)Chao, Bl{\"o}baum, and
  Kasiviswanathan]{chao2023interventional}
Patrick Chao, Patrick Bl{\"o}baum, and Shiva~Prasad Kasiviswanathan.
\newblock Interventional and counterfactual inference with diffusion models.
\newblock {\tt arXiv preprint arXiv:2302.00860}, 2023.

\bibitem[Strichartz(2003)]{strichartz2003guide}
Robert~S Strichartz.
\newblock \emph{A guide to distribution theory and Fourier transforms}.
\newblock World Scientific Publishing Company, 2003.

\bibitem[Mityagin(2015)]{mityagin2015zero}
Boris Mityagin.
\newblock The zero set of a real analytic function.
\newblock {\tt preprint arXiv:1512.07276}, 2015.

\end{thebibliography}

\paragraph{Acknowledgements}
Thanks to Robin Evans for correcting remarks on an earlier version. Part of this work was done while Philipp Faller was an intern at Amazon Research.

\newpage

\appendix
\section*{Appendix A}
\label{app:theorem}
In this appendix we provide the proof for Lemma~\ref{lem:non_linear_anm}.
Some ideas in this proof are inspired  by \cite{chao2023interventional}.
\begin{proof}[for Lemma \ref{lem:non_linear_anm} ]
	We proof the statement by contradiction. 
	Assume there is function $g:\cX\to \cZ$ and a noise variable $\hat N$ such that $g(X) + \hat N = Z$ and $X\ind \hat N$.
	First, we define a function
	$
	h_x(n_Y) := f_Z(f_Y(x) +  n_Y) - g(x)
	$
	and denote $W:= h_X(N_Y)$. Then we have
	\begin{displaymath}
		\hat N = h_X( N_Y) + N_Z = W + N_Z.
	\end{displaymath}
	Note that $N_X(N_Y)$ is a random variable and for all $x\in\cX$ also $N_x(N_Y)$ is one.
	As $h_X(N_Y)+N_Z\ind X$ and as $N_Z$ is shielded from $X$ we have  for all $x, x'\in\cX$ that 
	\begin{equation}
		\label{eq:h_x_equal_in_dist}
		h_X(N_Y) + N_Z \overset{d}{=} h_x(N_Y) + N_Z \overset{d}{=} h_{x'}(N_Y) + N_Z.
	\end{equation}
	As $N_Y\ind N_Z$ we also get $h_x(N_Y)\ind N_Z$ for all $x\in \cX$. This renders all terms in \cref{eq:h_x_equal_in_dist} sums of independent variables and we have
	\begin{displaymath}
		\phi_{h_x(N_Y)}(t)\phi_{N_Z}(t) = \phi_{h_{x'}(N_Y)}(t)\phi_{N_Z}(t),
	\end{displaymath}
	for all $t\in \R$, where $\phi_V$ denotes the characteristic function of a random variable $V$.
	As we assumed $\phi_{N_Z}$ to be non-zero almost everywhere,  the equality ${\phi_{h_x{N_Y}}(t) = \phi_{h_{x'}{N_Y}}(t)}$ holds for almost all $t\in R$. 
	
	Thus the densities of $h_x(N_Y)$ also agree for all $x\in \cX$. 
	In other words, we also have $W\ind X$.
	
	As $f_Z$ is invertible, the function $n_Y \mapsto h_x(n_Y)$ is invertible as well.
	We may also  assume w.l.o.g. that $N_Y$ is uniformly distributed, as any other continuos noise with strictly increasing c.d.f. could be achieved by a invertible transformation of the uniform $N_Y$. 
	Let $\mathcal{W}$ be the support of $W$.
	Using the density change formula, the fact that $W\ind X$ and the monotonicity of $f_Z$, we can write for all $w\in \mathcal{W}$ and $x\in\cX$
	\begin{align*}
		p_{W}(w) &= p_{W\mid X}(w\mid x) = p_{N_Y}(h^{-1}_x(w))\left|\frac{d}{d w}(h^{-1}_x(w))\right| = c_1 \left|\frac{d}{d w}(h^{-1}_x(w))\right| 
		\\
		&= c_1 \frac{d}{d w}(h^{-1}_x(w))
	\end{align*}
	for some constant $c_1$, where $p_{N_Y}$ denotes the uniform density of $N_Y$. 
	As this holds for all $x\in\cX$, the factor $\frac{d}{dw}h^{-1}_x(w)$ must be constant in $x$.
	Moreover, as all $h_x^{-1}$ have the same derivatives for all $w\in \mathcal{W}$, the functions $h^{-1}_x$ can only differ by an additive term, which may depend on $x$.  I.e.
	\begin{displaymath}
		h^{-1}_x(w) = h^{-1}_{x'}(w) + c(x),
	\end{displaymath} 
	By fixing one of these functions as $h^* := h_x$ we can express all $h_x$ via
	\begin{displaymath}
		h_{x'}(n_y) = h^*(n_Y - c(x')).
	\end{displaymath}
	Further, as $W$ and $X$ are independent, the support of $W$ must be the same for all $w\in \mathcal{W}$.
	We denote for some function $f$ and set $S$ the image $f(S) := \{f(s) \mid s\in S\}$. So we get
	\begin{align*}
		h^*(\cW + c(x)) &= h^*(\cW + c(x'))\\
		\cW + c(x) &=\cW + c(x'),
	\end{align*}
	where we used the invertibility of $h^*=h_x$.
	Clearly, $c(x)$ must be constant then.
	Now finally, we have that $h_x(n_Y) = f_Z(f_Y(x) + n_Y) - g(x)$ is constant in $x$, i.e. for all $x\in\cX, n_Y\in\mathcal{N}_Y$
	\begin{align*}
		\frac{\partial}{\partial x} h_x(n_Y) &= 0\\
		\frac{\partial}{\partial x} f_Z(f_Y(x) + n_Y) - \frac{\partial}{\partial x} g(x) &= 0\\
		\frac{\partial}{\partial x} f_Z(f_Y(x) + n_Y) &= \frac{\partial}{\partial x} g(x),
	\end{align*}
	which is a contradiction to \cref{eq:condition_non_linear_scm}.
	
\end{proof}

Note that the non-zero characteristic function of the noise terms, required for Lemma~\ref{lem:non_linear_anm}, are for example provided by positive noise terms or noise terms with compact support
\footnote{Then the Paley-Wiener theorems (e.g. \cite{strichartz2003guide}, Theorem 7.2.1 and Theorem 7.2.4) imply that the characteristic function is analytic. The zeros of analytic functions have Lebesgue-measure zero \citep{mityagin2015zero} if the functions are not constantly zero.}.
Especially, this holds true for uniformly distributed noise.

Further assume e.g. that $f_Z$ and $f_Y$ are analytic functions in finitely many real parameters.
I.e. for $a\in \R^k, b\in \R^{k'}$ denote them as $f^a_Z$ and $f^b_Y$.
Then the set of functions violating \cref{eq:condition_non_linear_scm} has Lebesgue-measure zero in the parameter space unless all functions violate it.
This can be seen as follows: 
If $f^a_Z$ and $f^b_Y$ are analytic in $a, b, x$ and $y$, then the map
\begin{equation}
	\label{eq:map_analytic_zero}
	(a, b, x, n_Y, n'_Y) \mapsto \frac{\partial}{\partial x}f^a_Z(f^b_Y(x) +  n_Y) - \frac{\partial}{\partial x} f^a_Z(f^b_Y(x) + n'_Y)
\end{equation} 
is analytic as well.
As setting of the model parameters $a, b$ and values $x, n_Y, n'_Y$ that falsify the inequality in Eq.~\ref{eq:condition_non_linear_scm} are a zero of the mapping in Eq.~\ref{eq:map_analytic_zero}.
So either the mapping in Eq.~\ref{eq:map_analytic_zero} is constant zero, or Eq.~\ref{eq:condition_non_linear_scm} holds only for a subset with Lebesgue-measure zero in the space $\R^k\times \R^{k'}\times \cX\times \mathcal{N}_Y\times \mathcal{N}_Y$\citep{mityagin2015zero}. 
So Eq.~\ref{eq:condition_non_linear_scm} holds for almost all functions indexed by $(a, b)$ or none of them.
For example, the set of linear functions is a set where all functions violate \cref{eq:condition_non_linear_scm}, while $f_Z(f_Y(x) + n_Y) = \exp(x + n_Y)$ fulfils it.


\section*{Appendix B}
\begin{figure}[htp]
	\centering
	\includegraphics[width=0.7\textwidth]{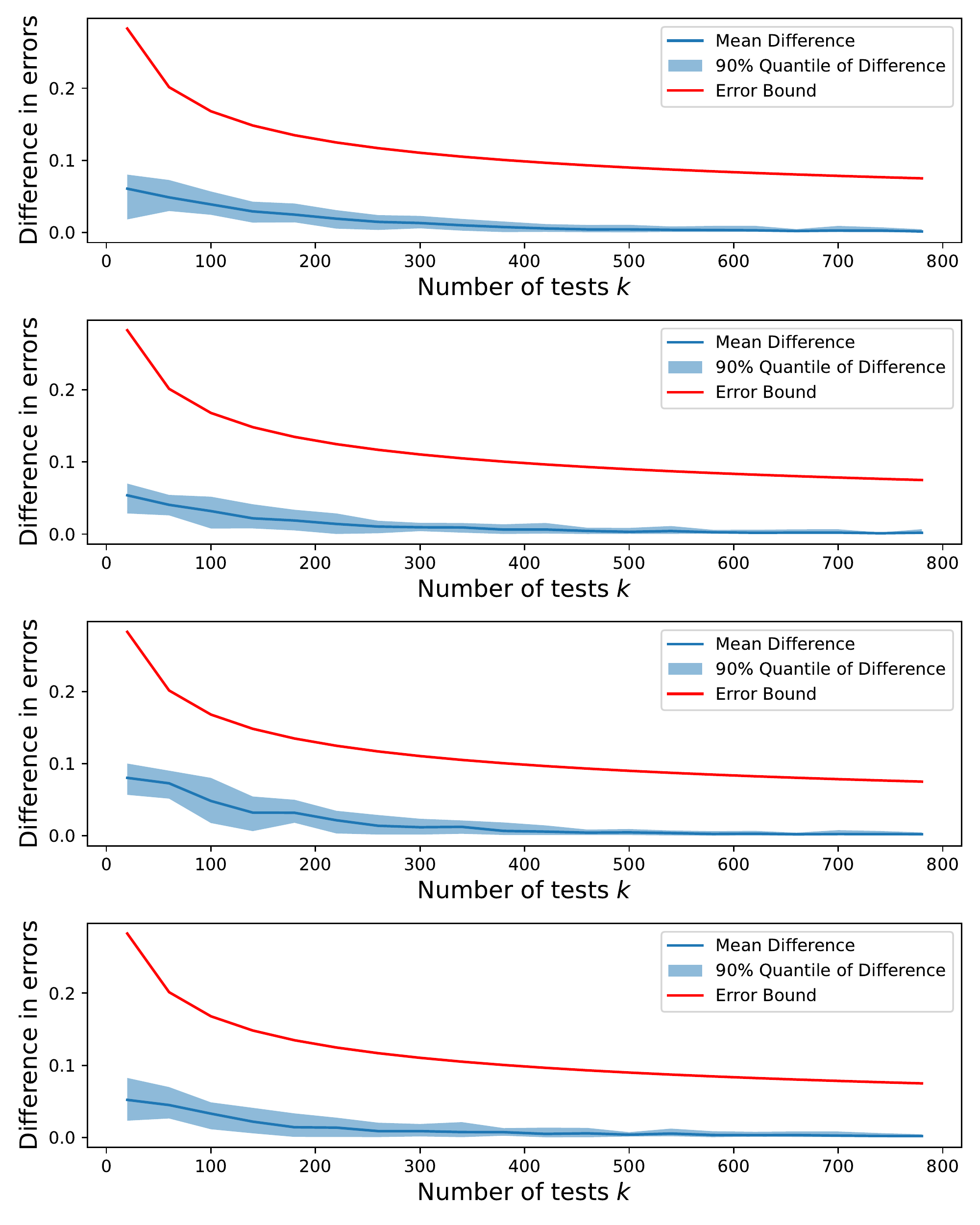}
	\caption{The difference between empirical error $\frac{1}{k} \sum_{i=1}^k|Q_T(D_i) - Q_G(D_i)|$ and expected error $\mathbb{E}|Q_T(D) - Q_G(D)|$ for graphs with $n=20$ nodes when $k$ tuples of variables are used to fit a polytree. These are additional simulated datasets}
\end{figure}

\end{document}